\documentclass{article}

\usepackage{microtype}
\usepackage{graphicx}
\usepackage{subcaption}
\usepackage{booktabs} 

\usepackage{hyperref}

\usepackage[preprint]{icml2026}

\usepackage{amsmath}
\usepackage{amssymb}
\usepackage{mathtools}
\usepackage{amsthm}
\usepackage{bbm}
\usepackage{multirow}

\usepackage[capitalize,noabbrev]{cleveref}

\theoremstyle{plain}
\newtheorem{theorem}{Theorem}[section]
\newtheorem{proposition}[theorem]{Proposition}

\theoremstyle{definition}

\theoremstyle{remark}

\usepackage[textsize=tiny]{todonotes}

\icmltitlerunning{Adaptive Uncertainty-Aware Tree Search for Robust Reasoning}

\begin{document}

\twocolumn[
  \icmltitle{Adaptive Uncertainty-Aware Tree Search for Robust Reasoning}



  \icmlsetsymbol{equal}{*}

  \begin{icmlauthorlist}
    \icmlauthor{Zeen Song}{iscas,ucas}
    \icmlauthor{Zihao Ma}{CSU}
    \icmlauthor{Wenwen Qiang}{iscas}
    \icmlauthor{Changwen Zheng}{iscas}
    \icmlauthor{Gang Hua}{dolby,XJ}
  \end{icmlauthorlist}

  \icmlaffiliation{iscas}{Institute of Software Chinese Academy of Sciences, Beijing, China}
  \icmlaffiliation{ucas}{University of the Chinese Academy of Sciences}
  \icmlaffiliation{dolby}{Multimodal Experiences Lab, Dolby Laboratories Inc}
  \icmlaffiliation{XJ}{Institute of Artificial Intelligence and Robotics, Xi’an Jiaotong University}
  \icmlaffiliation{CSU}{Electronic Information Science and Technology, Central South University}

  \icmlcorrespondingauthor{Wenwen Qiang}{qiangwenwen@iscas.ac.cn}

  \icmlkeywords{Machine Learning, ICML}

  \vskip 0.3in
]



\printAffiliationsAndNotice{}  

\begin{abstract}
  Inference-time reasoning scaling has significantly advanced the capabilities of Large Language Models (LLMs) in complex problem-solving. A prevalent approach involves external search guided by Process Reward Models (PRMs). However, a fundamental limitation of this framework is the epistemic uncertainty of PRMs when evaluating reasoning paths that deviate from their training distribution. In this work, we conduct a systematic analysis of this challenge. We first provide empirical evidence that PRMs exhibit high uncertainty and unreliable scoring on out-of-distribution (OOD) samples. We then establish a theoretical framework proving that while standard search incurs linear regret accumulation, an uncertainty-aware strategy can achieve sublinear regret. Motivated by these findings, we propose Uncertainty-Aware Tree Search (UATS), a unified method that estimates uncertainty via Monte Carlo Dropout and dynamically allocates compute budget using a reinforcement learning-based controller. Extensive experiments demonstrate that our approach effectively mitigates the impact of OOD errors.
\end{abstract}

\section{Introduction}
\begin{figure}[t]
    \centering
    \includegraphics[width=\linewidth]{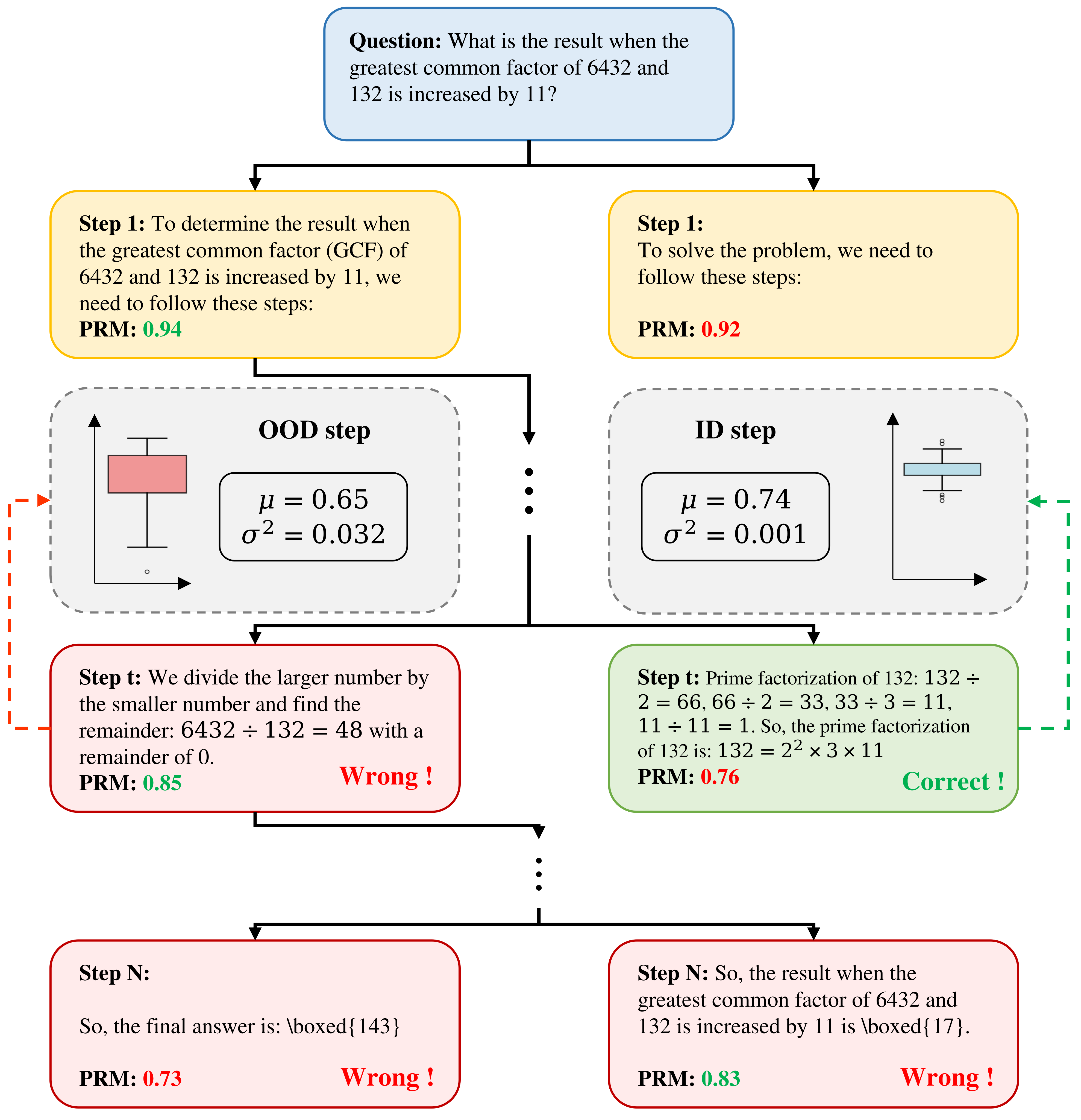}
    \caption{Visualization of a real beam search step on the MATH dataset. The diagram depicts a selection conflict where the PRM assigns a higher reward ($0.85$) to an incorrect reasoning step (left) containing a calculation error than to a correct step ($0.76$, right). Despite the high reward, the epistemic uncertainty (visualized via boxplots) reveals the unreliability of the incorrect node, which exhibits significantly higher score variance ($\sigma^2=0.032$) compared to the correct node ($\sigma^2=0.001$).}
    \label{fig:moti}
\end{figure}

In recent years, chain-of-thought (CoT) prompting has substantially improved the performance of large language models (LLMs) on complex reasoning tasks such as math problem solving, question answering, and multi-hop retrieval \cite{weiChainofthoughtPromptingElicits2022, kojimaLargeLanguageModels2022, yaoTreeThoughtsDeliberate2023}. To support more effective CoT reasoning, recent works have explored inference-time reasoning strategies that allocate more compute during inference \cite{openai2024learning, deepseek-aiDeepSeekR1IncentivizingReasoning2025, teamKimiK15Scaling2025,snellScalingLLMTestTime2024a, liuCan1BLLM2025, wuInferenceScalingLaws2024}. These approaches can be broadly divided into internal and external ones, with external methods attracting increasing attention due to their flexibility and ability to enhance performance without modifying the base model \cite{snellScalingLLMTestTime2024a, liuCan1BLLM2025, wuInferenceScalingLaws2024}.

The external reasoning framework typically consists of three components: a frozen policy model, a search algorithm, and a process reward model (PRM). It generates multiple candidate reasoning paths using the policy model guided by the search algorithm, and reranks these paths using the PRM to select the most promising one \cite{yaoTreeThoughtsDeliberate2023, bestaGraphThoughtsSolving2024}. However, a critical vulnerability in this framework lies in the PRM's capability. Since PRMs are typically trained on a fixed distribution of reasoning traces, they often exhibit high epistemic uncertainty when evaluating reasoning steps that deviate from their training data \cite{kotelevskiiRiskUncertaintyGenerating2024}. Neglecting such samples during search causes the PRM to yield overconfident errors, leading to the rejection of valid paths.

To systematically address this challenge, we perform both empirical and theoretical analyses. Empirically, we evaluate open-source PRMs (Math-Shepherd-PRM-7B and Qwen2.5-Math-PRM-7B) on reasoning traces generated by diverse policy models ( LLemma-7B and Qwen-2.5-Instruct-7B ) using the MATH500 dataset. By measuring search accuracy and Monte Carlo Dropout score variance, our experiments reveal two findings: (i) PRMs suffer from significant accuracy drops when scoring OOD traces compared to in-distribution ones, and (ii) this degradation is strongly correlated with spikes in epistemic uncertainty. We also provide a real example in Figure \ref{fig:moti}. Theoretically, we model the reasoning search under uncertainty and derive regret bounds for different strategies. We prove that traditional greedy approaches incur linear regret accumulation in the presence of OOD samples. Conversely, we demonstrate that incorporating uncertainty estimates into the selection mechanism reduces this to sublinear regret, providing a rigorous foundation for uncertainty-aware search algorithms.

Motivated by these theoretical insights, we propose a novel Uncertainty-Aware Tree Search framework designed to bridge the gap between theory and practice. Our method first estimates epistemic uncertainty via Monte Carlo Dropout. We then introduce a heuristic search strategy that leverages this uncertainty to selectively re-evaluate promising but uncertain candidates. To further enhance adaptability, we formalize the search budget allocation as a Markov Decision Process (MDP) and train an adaptive controller using reinforcement learning. This controller dynamically modulates key search hyperparameters to optimally balance exploration and exploitation under resource constraints.

To validate our theoretical analysis and the effectiveness of the proposed UATS strategy, we conduct extensive experiments on two challenging mathematical reasoning benchmarks: MATH-500 \cite{hendrycksmath2021, lightmanLetsVerifyStep2023} and AIME24 \cite{ai-mo2024}. We evaluate UATS under multiple frozen policy models, including Qwen 2.5 \cite{yang2024qwen2}, Llama 3.1 \cite{grattafiori2024llama}, and Llama 3.2 \cite{meta2024llama3}, and incorporate a diverse set of PRMs \cite{zhangLessonsDevelopingProcess2025, wangMathShepherdVerifyReinforce2024, skywork2024}. The results demonstrate that UATS consistently achieves higher accuracy than other external reasoning methods across different model combinations and compute budgets. These results provide strong empirical support for our theoretical analysis. Our contributions are summarized as follows:
\begin{itemize}
    \item We identify the critical issue of epistemic uncertainty in PRM-based external reasoning, revealing that distribution shifts between the policy and PRM significantly undermine search reliability.
    \item We provide a rigorous analysis combining empirical evidence of PRM degradation on OOD data with theoretical proofs demonstrating that uncertainty-aware search provably minimizes search regret compared to uncertainty-agnostic baselines.
    \item We propose the Uncertainty-Aware Tree Search (UATS), a method that integrates uncertainty estimation with a learnable budget controller. Extensive experiments demonstrate that our approach significantly outperforms standard baselines, achieving superior accuracy and compute efficiency.
\end{itemize}

\section{Related Works}

\textbf{Inference-time Reasoning}. CoT prompting is first proposed as a prompting technique that enables LLMs to decompose problems into intermediate steps \cite{weiChainofthoughtPromptingElicits2022}. Recently, the OpenAI o1 series \cite{openai2024learning} demonstrate that increasing the length of CoT during inference yields substantial performance gains on tasks like MATH \cite{hendrycksmath2021} and AIME \cite{ai-mo2024}. 

External reasoning improves the reasoning performance via sampling or search-based methods with fixed LLMs and an external verifier \cite{lightmanLetsVerifyStep2023,wuInferenceScalingLaws2024, snellScalingLLMTestTime2024a, yaoTreeThoughtsDeliberate2023, selAlgorithmThoughtsEnhancing2024, bestaGraphThoughtsSolving2024, zhangAutomaticChainThought2022, brownLargeLanguageMonkeys2024}. Specifically, Tree‑of‑Thoughts \cite{yaoTreeThoughtsDeliberate2023} explores a look‑ahead search tree of thought chunks with self‑evaluation to achieve large gains on planning‑style tasks. 
Snell et al. \cite{snellScalingLLMTestTime2024a} find that adaptive allocation of verifier‑guided search can beat a 14× larger model while using 4× less extra compute than best‑of‑N sampling.

\textbf{Process Reward Model}.
An essential component of external reasoning is the verifier that evaluates different reasoning paths. Verifiers are categorized into two types: Process Reward Models (PRMs) and Outcome Reward Models (ORMs) \cite{uesato2022solving}. PRMs assess the quality of a reasoning step given the question and partial reasoning trajectory, estimating the likelihood that the process will lead to a correct answer \cite{lightmanLetsVerifyStep2023, wangMathShepherdVerifyReinforce2024}. In contrast, ORMs provide a reward signal based on the final answer's correctness, given the full reasoning trace and output \cite{uesato2022solving,lightmanLetsVerifyStep2023}. Recent studies have shown that PRMs are generally more effective than ORMs in guiding search \cite{lightmanLetsVerifyStep2023, uesato2022solving, snellScalingLLMTestTime2024a}, and PRMs have become a widely adopted tool in external test-time reasoning frameworks \cite{liuCan1BLLM2025, xieSelfEvaluationGuidedBeam2023,snellScalingLLMTestTime2024a}. \cite{lightmanLetsVerifyStep2023} trains PRM on 800k human‑labeled reasoning steps. Math-Shepherd \cite{wangMathShepherdVerifyReinforce2024} automatically constructs step-level supervision and assigns scores based on the proportion of branches that reach the known correct answer.

\section{Preliminary}
In this section, we first formally define the fundamental workflow of problem-solving via external reasoning. Subsequently, we briefly introduce the training paradigm of PRMs. Finally, we analyze the uncertainty of PRM outputs from a Bayesian perspective.

\subsection{Problem Solving with External Reasoning}
We consider reasoning problems defined in natural language with fixed answers, where each question $q \in \mathcal{Q}$ is associated with a ground truth answer $a^* \in \mathcal{A}$. These elements are represented as sequences of tokens; specifically, $\mathcal{Q}, \mathcal{A} \subset \mathcal{V}^*$, where $\mathcal{V}$ denotes the token space and $\mathcal{V}^* \triangleq \bigcup_{L=0}^{\infty} \mathcal{V}^L$ represents the set of all possible token sequences. We assume that the question-answer pairs $(q, a^*)$ follow an unknown joint distribution $\mathcal{P}_{\text{data}}(q, a^*)$.

In general, these problems are difficult to answer directly and instead require a step-by-step reasoning process \cite{weiChainofthoughtPromptingElicits2022,lightmanLetsVerifyStep2023,openai2024learning}. Under this situation, an external reasoning system is designed to address this challenge. Specifically, the system consists of the following components: (i) an LLM $\pi_\theta:\mathcal{V}^*\to\Delta(\mathcal{V})$ with fixed parameters $\theta$ is used as the proposer to generate partial reasoning traces $h_k=(q,z_1,\dots,z_k)$ in an autoregressive manner, where each $z_k\in\mathcal{V}^*$ is a reasoning step and $\Delta(\mathcal{V})$ denotes the probability simplex over $\mathcal{V}$; (ii) a PRM $R_\phi:\mathcal{V}^*\to[0,1]$ with fixed parameters $\phi$ is used as the verifier to evaluate the reasoning traces $h_k$ generated by $\pi_\theta$; (iii) a search algorithm $\mathrm{Alg}(\cdot;\pi_\theta,R_\phi,C):\mathcal{Q}\to\mathcal{V}^*$ leverages the traces proposed by $\pi_\theta$, the scores provided by $R_\phi$, and a computation budget $C\in\mathbb{N}^+$ to guide, filter, and ultimately output a selected reasoning trace $\hat{h}$; and (iv) an answer extractor $\mathcal{E}:\mathcal{V}^*\to\mathcal{A}$ extracts the final answer from the selected reasoning trace, namely $\hat{a}=\mathcal{E}(\hat{h})$.

In this context, existing research primarily focuses on designing search algorithms that maximize the accuracy of the final answer \cite{snellScalingLLMTestTime2024, wuInferenceScalingLaws2024}. Formally, given a dataset $\mathcal{D}_{\mathrm{tr}} = \{(q_i, a^*_i)\}_{i=1}^{N_{\text{train}}}$ where the samples are drawn from the distribution $\mathcal{P}_{\text{data}}(q, a^*)$, the goal is to find a search algorithm that maximizes the expected accuracy. Consequently, the objective function is defined as:
\begin{equation}
\label{eq:obj}
\max_{\mathrm{Alg}} \mathbb{E}_{(q, a^*) \sim \mathcal{P}_{\text{data}}} \left[ \mathbbm{1}_{\left( \mathcal{E}\big(\mathrm{Alg}(q; \pi_\theta, R_\phi, C)\big) = a^* \right)} \right],
\end{equation}
where the mathematical term $\mathbbm{1}_{(\cdot)}$ denotes the indicator function that equals $1$ if the extracted answer matches the ground-truth answer $a^*$, and $0$ otherwise.

\subsection{PRM: Process Reward Model}
\label{sec:prm}
The PRM \(R_\phi\) plays a central role in the external reasoning system. A PRM is typically implemented by appending a linear prediction head to another LLM (different from $\pi_\theta$) and then fine-tuning the entire network on supervised training data \cite{uesato2022solving, lightmanLetsVerifyStep2023, zhangLessonsDevelopingProcess2025,wangMathShepherdVerifyReinforce2024}. The dataset $\mathcal{D}_{\mathrm{PRM}} = \{(q_i, \{(z_{i,t}, y_{h_{i,t}})\}_{t=1}^{T_i})\}_{i=1}^{n}$ for training PRM consists of $n$ questions and corresponding reasoning steps $\{z_{i,t}\}_{t=1}^{T_i}$, where \(q_i\) denotes the \(i\)-th question, \(z_{i,t}\) represents the $t$-th reasoning step for question \(q_i\), $h_{i,t}=(q_i,z_{i,1},\dots,z_{i,t})$ is the reasoning trace and \(y_{h_{i,t}} \in \{0,1\}\) is the quality label for \(h_{i,t}\), with $1$ indicating ``good" and $0$ indicating ``bad". The label collection process can be found in \cite{lightmanLetsVerifyStep2023, wangMathShepherdVerifyReinforce2024}. Given the above training data, the PRM training objective is formulated as:
\begin{equation}
\label{eq:prm}
\begin{aligned}
\min_\phi \quad
& \sum_{i=1}^N \sum_{t=1}^{T_i}
\Big(
y_{h_{i,t}} \log R_\phi(h_{i,t}) \\
& \quad + (1 - y_{h_{i,t}})
\log \big(1 - R_\phi(h_{i,t})\big)
\Big).
\end{aligned}
\end{equation}
After training, the parameters of PRM are frozen. Within external reasoning, the trained PRM is expected to assign reliable scores to the reasoning steps produced by the frozen policy $\pi_\theta$ when confronted with unseen questions. 

\subsection{Prediction Uncertainty in PRMs}
\label{sec:motivation}
Intuitively, the score of reasoning steps produced by the PRM directly influences the selection of reasoning traces, thereby affecting the performance of external reasoning systems under Equation \ref{eq:obj}. As detailed in Section \ref{sec:prm}, the PRM is acquired through supervised fine-tuning on a finite set of training samples. From a Bayesian perspective, the parameters of PRM after minimizing Equation~\ref{eq:prm} follow a posterior distribution $p(\phi\mid \mathcal{D}_{\mathrm{PRM}})$. Consequently, the predictive uncertainty of the PRM can be characterized using the law of total variance:
\begin{equation}
\label{eq:uncertainty_decomp}
\resizebox{0.9\linewidth}{!}{$
\begin{aligned}
\underbrace{\mathrm{Var}(y \mid h, \mathcal{D}_{\mathrm{PRM}})}_{\text{Total Uncertainty}}
&=
\underbrace{
\mathbb{E}_{\phi \sim p(\phi \mid \mathcal{D}_{\mathrm{PRM}})}
\big[
\mathrm{Var}(y \mid h, \phi)
\big]
}_{\text{Aleatoric Uncertainty}} \\
&\quad+
\underbrace{
\mathrm{Var}_{\phi \sim p(\phi \mid \mathcal{D}_{\mathrm{PRM}})}
\big(
\mathbb{E}[y \mid h, \phi]
\big)
}_{\text{Epistemic Uncertainty}} .
\end{aligned}
$}
\end{equation}
The first term corresponds to the aleatoric uncertainty, which captures the intrinsic randomness of the outcome $y$. In the context of PRM, such uncertainty arises from the fact that a given $h$ may be extendable to multiple plausible future reasoning steps and final answers with different outcomes. Importantly, this source of uncertainty is independent of how well the PRM has been trained. 

The second term captures the epistemic uncertainty, which measures the variability in predictions caused by uncertainty over the model parameters. When a reasoning trace $h$ is highly similar to samples in $\mathcal{D}_{\mathrm{PRM}}$, all parameter configurations $\phi$ that achieve a good fit to Equation~\ref{eq:prm} are effectively constrained to produce consistent predictions, resulting in a small variance. In contrast, when $h$ lies far from the support of $\mathcal{D}_{\mathrm{PRM}}$, the training objective in Equation~\ref{eq:prm} provides little constraint on the behavior of the PRM at $h$, allowing different plausible parameter values to yield different predictions. Consequently, epistemic uncertainty quantifies how far a given reasoning trace deviates from the PRM training data \cite{kochenderfer2022algorithms}. When evaluating reasoning traces that deviate significantly from $\mathcal{D}_{\mathrm{PRM}}$, the PRM may exhibit high epistemic uncertainty.

Driven by the above analysis, we aim to answer the following questions: (i) How do existing PRMs perform when encountering reasoning traces that deviate from their training data? Can they still provide accurate assessments? (ii) If the PRM exhibits high uncertainty, what impact does this have on the accuracy of the external reasoning system?

\section{Problem Analysis}
\label{sec:theory}
In this section, we conduct a systematic analysis to answer the questions raised above. Specifically, we present empirical results revealing that current PRMs exhibit high uncertainty and suffer from severe performance degradation when evaluating out-of-distribution (OOD) reasoning steps. Subsequently, we establish a theoretical framework and prove that search algorithms ignoring this uncertainty incur unavoidable search regret. Furthermore, we demonstrate that an uncertainty-aware search process can provably minimize this regret, thereby providing a theoretical basis for subsequent algorithm design.

\subsection{Empirical Analysis}
To verify our hypothesis regarding PRM behavior under distribution shifts, we investigate the performance and prediction uncertainty of different PRMs when evaluating reasoning traces generated by various policy models.

We select two representative open-source PRMs: Math-Shepherd-PRM-7B \cite{wangMathShepherdVerifyReinforce2024} and Qwen2.5-Math-PRM-7B \cite{yangQwen25MathTechnicalReport2024}. For the policy models, we employ LLemma-7B and Qwen-2.5-Instruct-7B, creating combinations of PRMs and policies that vary in their underlying data distributions. The evaluation is conducted on a subset of 50 randomly sampled questions from the MATH500 dataset. We employ Beam width as the search strategy, configuring the beam width as 4 and number of final candidate answers as 8. We record two primary metrics: (1) \textbf{Accuracy}, averaged over 10 independent runs for each PRM-Policy pair to ensure statistical stability; and (2) \textbf{Score Variance}, which serves as a proxy for epistemic uncertainty, calculated by measuring the variance of PRM scores for sampled reasoning traces via Monte-Carlo Dropout.

\begin{figure}[tb]
    \centering
    \includegraphics[width=\linewidth]{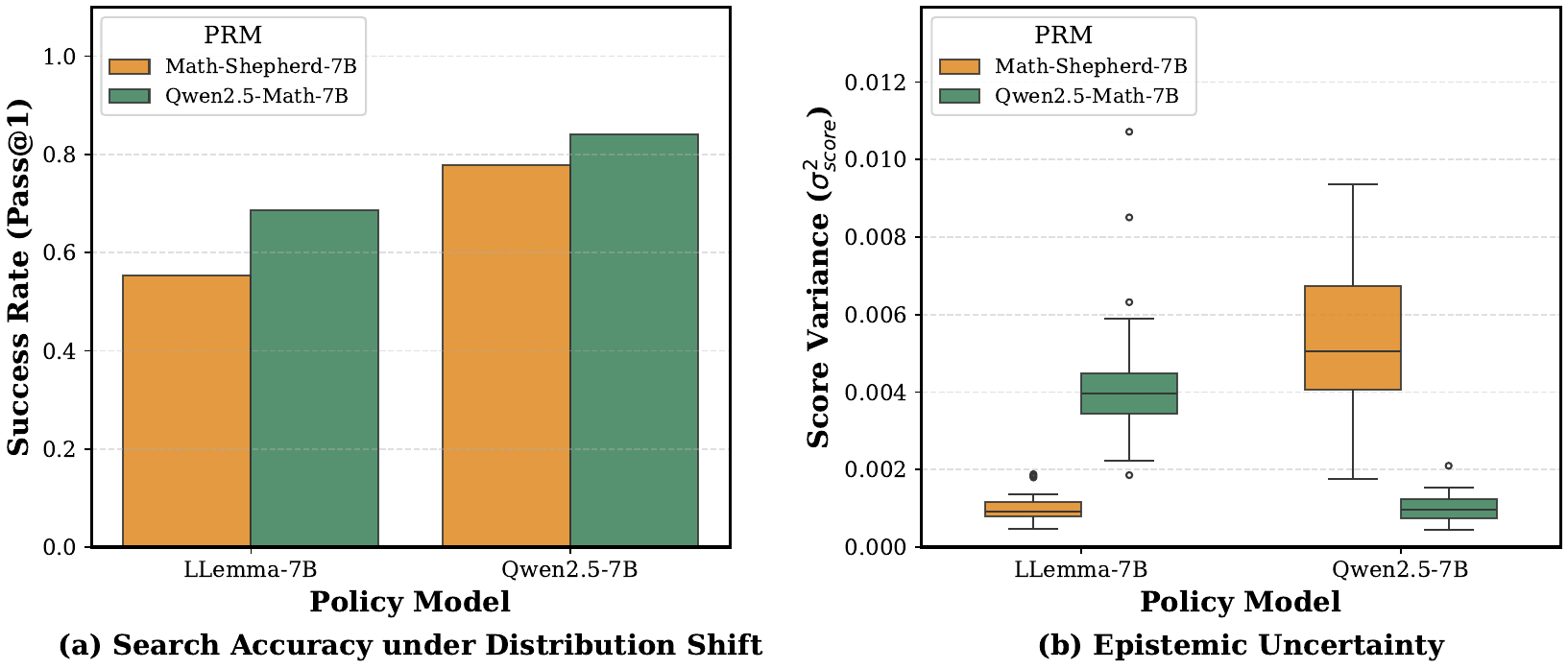}
    \caption{(a) Search accuracy across different PRM-Policy pairs. (b) The distribution of score variance across Monte Carlo Dropout for different PRM-Policy pairs}
    \label{fig:empirical}
\end{figure}

As shown in Figure \ref{fig:empirical}, PRMs exhibit varying robustness on fixed policies. Crucially, they demonstrate significantly higher score variance when evaluating traces from distributionally distinct policy models. Figure \ref{fig:moti} illustrates a real-world failure on the MATH dataset using Qwen-2.5-7B-Instruct and Math-Shepherd-PRM-7B. The PRM confidently misranks an incorrect OOD step ($R_\phi=0.85$) above a correct ID step ($R_\phi=0.76$), leading to search error. Epistemic uncertainty exposes this risk: the OOD step exhibits significantly higher variance ($\sigma^2=0.032$) than the robust ID step ($\sigma^2=0.001$). We provide more cases of unreliable scores in Appendix \ref{app:quali} in Figure \ref{fig:high_wrong} and \ref{fig:low_right}.

\subsection{Theoretical Analysis}
In this section, we provide a unified theoretical analysis that connects the uncertainty in PRMs with decision-making strategies for external reasoning. Then, we discuss how these theoretical observations inform the modeling of external reasoning methods.

\textbf{Influence of Uncertainty} 
In external reasoning, the objective is to iteratively select reasoning prefixes that maximize the true reward, thereby improving final answer correctness. To study the impact of uncertainty on this process, we consider a simplified setting.

At each step \(t\in\{1,\dots,T\}\), the proposer generates a set of \(M\) candidate reasoning traces \(\mathcal{H}_t=\{h_{t,1},\dots,h_{t,M}\}\). A pretrained PRM \(R_\phi(\cdot)\) with parameters \(\phi\sim p(\phi\mid\mathcal{D}_{\mathrm{PRM}})\) is used to score candidates, while the true reward function \(R^*(\cdot)\) is unobserved. We model the PRM score as \(R_\phi(h)=R^*(h)+\xi(h)\), where \(\xi(h)\) captures epistemic uncertainty.

At each step, the algorithm selects \(\hat h_t = \arg\max_{h\in\mathcal{H}_t} R_\phi(h)\) and continues generation from the selected prefix. This procedure is repeated for \(T\) steps. We consider the practical regime in which the point estimate \(R_\phi\) is directly used for ranking. Let \(\mathcal{O}_t\) denote the event that \(\mathcal{H}_t\) contains at least one OOD trace, and assume \(\mathbb{P}(\mathcal{O}_t)=\varepsilon\) for all \(t\). We now state the main result under this setting:
\begin{proposition}[Linear degradation under epistemic uncertainty]
\label{prop:linear_degrade_xi_compact}
Consider the above scenario. Suppose (i) $\mathbb{P}(\mathcal{O}_t)=\varepsilon$ for all $t$; 
(ii) under a fixed continuation policy, if $\hat h_t\neq h_t^*$ then the final success probability decreases by at least $\underline{\Delta}$; 
and (iii) there exists a lower bound $\rho\in(0,1]$ such that $\mathbb{P}(\hat h_t\neq h_t^* \mid \mathcal{O}_t)\ \ge\ \rho,\
\mathbb{P}(\hat h_t\neq h_t^* \mid \mathcal{O}_t^c)=0$. Then
\begin{equation}
\mathbb{E}[\mathrm{Acc}_T]
\ \le\
R^*(h_0)\;-\;T\,\varepsilon\,\underline{\Delta}\,\rho,
\label{eq:linear_bound_rho}
\end{equation}
where $\mathrm{Acc}_T$ is defined as $\mathrm{Acc}_T \triangleq R^*(\hat h_T)$.
\end{proposition}
The proof is provided in Appendix \ref{app:proof_1}. When epistemic uncertainty increases, $U_{\mathrm{ood}}$ becomes larger, so the noise gap $\xi(h_t^{(2)})-\xi(h_t^*)$ is more likely to exceed the true gap $\Delta_t$, making the PRM ranking unreliable at OOD steps. Since OOD appears with probability $\varepsilon$ at each iteration, the induced mistake probability accumulates over $T$ steps, leading to a linear degradation in the final expected correctness.

\textbf{Uncertainty-Aware Regime} 
We now consider a second scenario in which epistemic uncertainty is explicitly taken into account during external reasoning. Following the same settings as above, instead of relying on a single point estimate, we draw $K_t$ i.i.d.\ samples $\phi_1,\dots,\phi_{K_t}\sim p(\phi\mid\mathcal{D}_{\mathrm{PRM}})$, where $K_t$ is allowed to increase with $t$. For each candidate $h\in\mathcal{H}_t$, we compute the empirical posterior mean $\bar R_t(h)\triangleq \frac{1}{K_t}\sum_{k=1}^{K_t} R_{\phi_k}(h).$ We then construct an Upper Confidence Bound (UCB) $U_t \triangleq \sqrt{\frac{2\ln t}{K_t}},$ and estimate the true reward by $\hat R_t(h)\triangleq \bar R_t(h)+U_t.$ The selection is based on the estimated reward. Then we have:
\begin{proposition}[Sublinear degradation under uncertainty-aware selection]
\label{prop:sublinear_degrade_ucb}
Consider the second scenario with uncertainty estimation. Suppose (i) $\mathbb{P}(\mathcal{O}_t)=\varepsilon$ for all $t$; (ii) the PRM estimators are unbiased, i.e., $\mathbb{E}_{\phi}[\bar R_t(h)]=R^*(h)$; and (iii) the selection policy $\hat h_t$ follows the UCB criterion: $\hat h_t \in \arg\max_{h\in\mathcal{H}_t} \left( \bar R_t(h) + \sqrt{\frac{2\ln t}{K_t}} \right)$. Then, we have the lower bound:
\begin{equation}
\label{eq:sublinear_bound}
\mathbb{E}[\mathrm{Acc}_T]\ \ge R^*(h_0) - 2\varepsilon\sum_{t=2}^T \sqrt{\frac{2\ln t}{K_t}} - O(1).
\end{equation}
In particular, if the sample budget scales as $K_t=\Omega(t)$, then
\begin{equation}
\mathbb{E}[\mathrm{Acc}_T]\ \ge R^*(h_0) - O\big(\varepsilon\sqrt{T\ln T}\big).
\end{equation}
\end{proposition}
The proof is provided in Appendix \ref{app:proof_2}. Proposition~\ref{prop:sublinear_degrade_ucb} demonstrates that by explicitly modeling epistemic uncertainty, the accumulation of errors shifts from a linear regime ($O(T)$) to a sublinear regime ($O(\sqrt{T})$). The UCB term ensures that even when OOD traces appear with probability $\varepsilon$, the potential loss at each step decreases over time as the sample size $K_t$ grows.

\begin{figure*}
\centering
    \includegraphics[width=\linewidth]{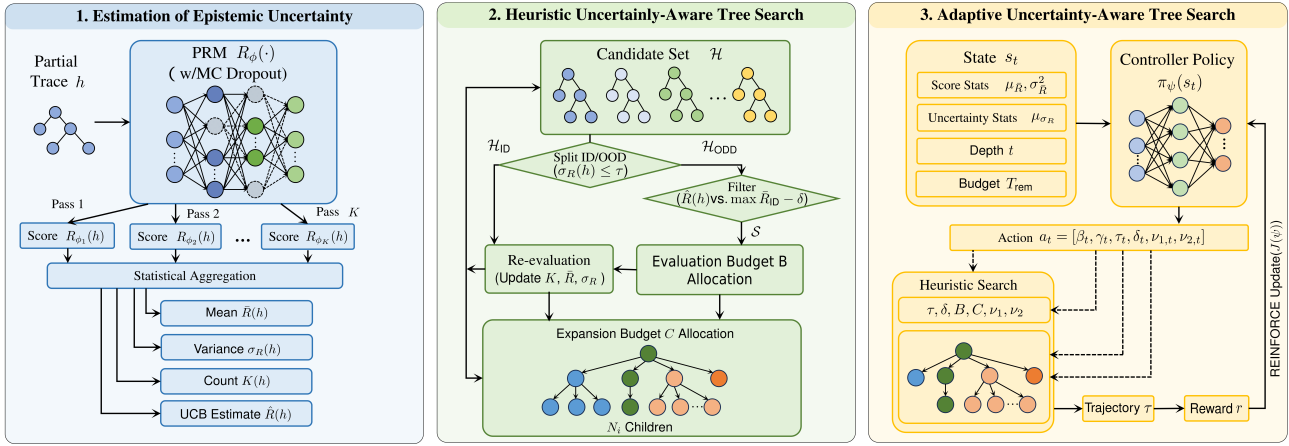}
    \caption{Overview of the proposed Uncertainty-Aware Tree Search (UATS) framework. The pipeline consists of three phases: (1) Approximate the PRM's uncertainty using Monte Carlo Dropout. (2) Use uncertainty thresholds ($\tau$) and optimism margins ($\delta$) to selectively re-evaluate ambiguous nodes and allocate expansion budgets. (3) Observes the search state $s_t$ and dynamically outputs action vectors $a_t$ to modulate budget factors and temperature parameters.}
\end{figure*}

\subsection{Theoretical Implications for Method Design}
\label{sec:inspire}
The theoretical analysis above reveals a fundamental limitation of uncertainty-agnostic external reasoning. In Scenario 1, epistemic uncertainty induces a non-negligible probability of mis-ranking candidate reasoning traces, and such errors accumulate linearly with the reasoning horizon. In contrast, Scenario 2 demonstrates that explicitly accounting for uncertainty mitigates the adverse effect of OOD candidates, yielding only a sublinear degradation in the final expected correctness. These results motivate the development of an uncertainty-aware external reasoning framework.

Translating this principle into a practical algorithm, presents two modeling challenges. (i) Uncertainty-aware selection requires access to multiple samples from the parameter posterior $p(\phi\mid\mathcal{D}_{\mathrm{PRM}})$. In practice, obtaining multiple parameter instances that are both computationally efficient and sufficiently diverse to faithfully reflect epistemic uncertainty is non-trivial. (ii) Incorporating uncertainty introduces an explicit trade-off between evaluation accuracy and computational cost. Given a fixed inference-time budget, it remains unclear how resources should be allocated between spending more computation to obtain a more accurate estimate of $R^*(h)$ via repeated PRM scoring and expanding the search by generating additional child nodes, especially when candidates differ in their uncertainty and reliability.

In the next section, we propose an uncertainty-aware external reasoning framework that addresses these challenges in a unified and principled manner.

\section{Methodology}
To address the challenges above, we propose an uncertainty-aware external reasoning framework that integrates uncertainty estimation, heuristic control, and learned adaptation in a unified pipeline: we first approximate epistemic uncertainty via Monte Carlo Dropout by sampling multiple stochastic PRM forward passes to obtain both a reward estimate and an uncertainty measure; we then design a heuristic search that allocates computation between repeated PRM scoring and expanding additional child nodes under a fixed inference-time budget; finally, we treat this heuristic as an initialization and learn a dynamic controller via reinforcement learning that adaptively outputs budget allocation during search, enabling more effective reasoning.

\subsection{Estimation of Epistemic Uncertainty}
\label{sec:estimate}
Given a partial reasoning trace $h$, to approximate the posterior variability of PRM predictions need drawing multiple parameter instances from $p(\phi\mid\mathcal{D}_{\mathrm{PRM}})$. Concretely, we sample $\phi_1,\dots,\phi_K\sim p(\phi\mid\mathcal{D}_{\mathrm{PRM}})$ and evaluate the same trace with $R_{\phi_1}(h),\dots,R_{\phi_K}(h)$, which yields a set of stochastic scores reflecting epistemic uncertainty.

In practice, directly sampling network parameters from the true posterior is intractable for modern PRMs. We therefore adopt Monte Carlo Dropout \cite{gal2016dropout} as an efficient approximation. Specifically, we enable dropout layers during inference and perform $K$ stochastic forward passes of the PRM on the same input $h$, each pass corresponding to a different dropout mask and thus a different effective parameter realization. Denoting the resulting scores by $\{R_{\phi_k}(h)\}_{k=1}^K$, we compute the empirical mean score and the empirical uncertainty as the sample variance:
\begin{equation}
\resizebox{0.9\linewidth}{!}{
$
\displaystyle
\bar{R}(h)\triangleq \frac{1}{K}\sum_{k=1}^{K} R_{\phi_k}(h),\ \sigma_R^2(h)\triangleq \frac{1}{K-1}\sum_{k=1}^{K}\big(R_{\phi_k}(h)-\bar{R}(h)\big)^2.
$
}
\end{equation}
Here, $\bar{R}(h)$ serves as a posterior-averaged estimate of the PRM score, while $\sigma_R(h)$ captures the epistemic dispersion across sampled parameter instances.

To incorporate an explicit optimism bonus, we further construct an upper-confidence estimate at reasoning step $t$,
\begin{equation}
\hat{R}(h)\triangleq \bar{R}(h) + \alpha \sqrt{\frac{2\ln t}{K}},
\label{eq:ucb_estimator}
\end{equation}
where $\alpha$ is a exploration parameter. Intuitively, increasing $K$ reduces the uncertainty of $\bar{R}(h)$ by averaging multiple stochastic evaluations, while the term $\sqrt{\frac{2\ln t}{K}}$ encourages exploration by accounting for finite-sample uncertainty in the estimate. In the subsequent subsections, we use $(\bar{R}(h),\sigma_R(h),K(h),\hat{R}(h))$ as the key signals to guide search-time budget allocation.

\subsection{Heuristic Uncertainty-Aware Tree Search}
\label{sec:heuristic}
To address the trade-off between more accurate PRM evaluation and broader tree expansion under a fixed compute budget, we introduce a Heuristic Uncertainty-Aware Tree Search procedure (H-UATS). The method uses $(\bar{R}(h),\sigma_R(h),K(h),\hat{R}(h))$ estimated above to (i) decide which candidates deserve additional PRM evaluations and (ii) allocate child-node expansion budget.

Given a candidate set of partial traces $\mathcal{H}$ at a search step, we first perform an initial evaluation pass by running $K_0$ Monte Carlo Dropout forward passes for each $h\in\mathcal{H}$. We then compute $(\bar{R}(h),\sigma_R(h),K(h),\hat{R}(h))$ following Section \ref{sec:estimate}, where $K(h)=K_0$ after initialization. Next, we partition candidates into two subsets based on epistemic uncertainty: using a threshold on $\sigma_R(h)$, we obtain ID subset $\mathcal{H}_{\mathrm{ID}} \triangleq \{h\in\mathcal{H}: \sigma_R(h)\le \tau\}$ and OOD subset $\mathcal{H}_{\mathrm{OOD}} \triangleq \{h\in\mathcal{H}: \sigma_R(h)>\tau\}$, where $\tau$ is a fixed or step-dependent threshold.

The core of the heuristic lies in selectively investing the additional PRM evaluation budget $B$ on uncertain candidates that show promise. Specifically, an OOD candidate $h_i\in\mathcal{H}_{\mathrm{OOD}}$ is selected for re-evaluation only if its optimistic estimate is competitive with the best ID candidate, satisfying $\hat{R}(h_i) \ge \max_{h\in\mathcal{H}_{\mathrm{ID}}}\bar{R}(h) - \delta$, where $\delta$ is a retention margin. For the subset of OOD candidates $\mathcal{S}$ that meet this criterion, we allocate the additional evaluation budget $B$ according to a softmax distribution over their optimistic scores, modulated by a temperature parameter $\nu_1$:
\begin{equation}
K_i \;=\; \mathrm{Round}\!\left(
B\cdot\frac{\exp(\hat{R}(h_i)/\nu_1)}{\sum_{h_j\in\mathcal{S}}\exp(\hat{R}(h_j)/\nu_1)}
\right).
\label{eq:eval_budget_alloc}
\end{equation}
A lower $\nu_1$ concentrates the re-evaluation budget on the few most promising OOD candidates to aggressively reduce their uncertainty, while a higher $\nu_1$ distributes the budget more uniformly for broader verification.

Following the additional evaluations and the subsequent update of the posterior statistics, we proceed to allocate the child-node expansion budget $C$. Adopting the REBASE principle \cite{wuInferenceScalingLaws2024}, we assign expansion counts proportional to the exponentiated posterior mean scores, controlled by a second temperature parameter $\nu_2$:
\begin{equation}
N_i \;=\; \mathrm{Round}\!\left(
C\cdot\frac{\exp(\bar{R}(h_i)/\nu_2)}{\sum_{h_j\in\mathcal{H}}\exp(\bar{R}(h_j)/\nu_2)}
\right).
\label{eq:child_alloc}
\end{equation}
We then expand each candidate $h_i$ to generate $N_i$ children.

This heuristic achieves a principled separation of concerns. For OOD candidates, the algorithm first uses the optimistic estimate $\hat{R}(h)$ to identify those that could plausibly compete with the best ID trace, and selectively allocates additional evaluation budget to reduce epistemic error and avoid premature elimination. Conversely, uncertain candidates with insufficient potential are deprioritized early to conserve computation. After re-evaluation, child-node expansion follows a score-proportional REBASE allocation based on $\bar{R}(h)$, which balances exploitation of high-scoring traces and continued exploration across the candidate set.

\begin{figure*}[t]
    \centering
    \includegraphics[width=\linewidth]{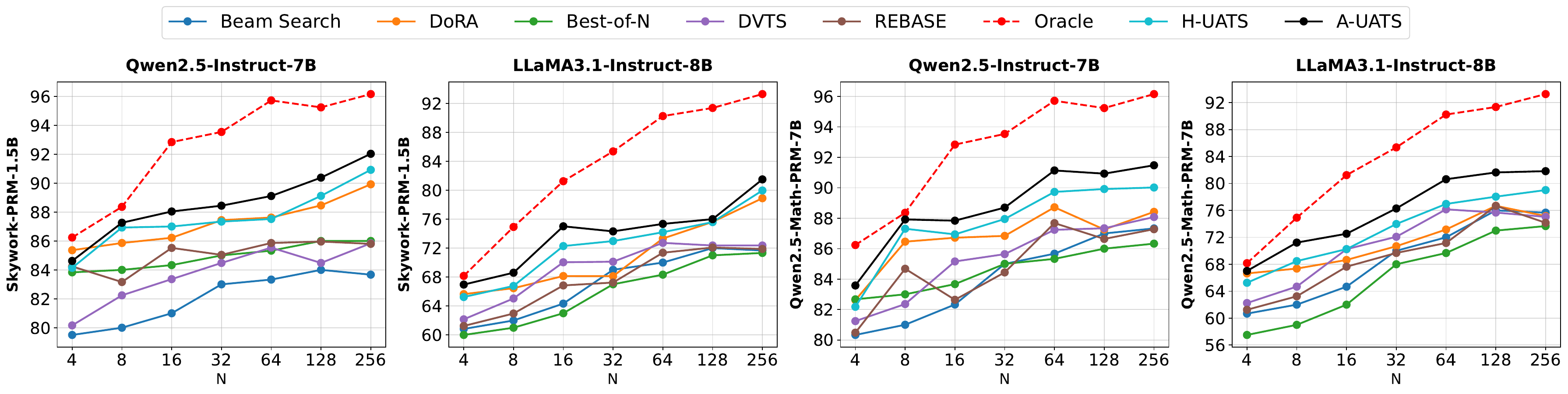}
    \caption{Accuracy comparison on MATH-500 across different Policy Model and PRM combinations. The x-axis represents the number of candidate paths ($N$) on a logarithmic scale. Our methods, H-UATS and A-UATS, consistently outperform standard baselines. }
 
    \label{fig:math}
\end{figure*}

\subsection{Adaptive Uncertainty-Aware Controller}
\label{sec:adaptive}
While the H-UATS effectively leverages uncertainty, it relies on static hyperparameters, including the budgets $B, C$, thresholds $\tau, \delta$, and temperatures $\nu_1, \nu_2$, which remain fixed throughout the reasoning process. However, the complexity of reasoning is inherently dynamic: some steps are straightforward and require minimal verification, while others are ambiguous and demand extensive exploration. A static allocation policy inevitably leads to inefficiency. To address this, we propose Adaptive Uncertainty-Aware Tree Search (A-UATS), parametrized by a controller $\psi$ to dynamically modulate these parameters.

Formally, we define the dynamic budget regulation problem as a Markov Decision Process (MDP). The state space $\mathcal{S}$ is designed to capture the statistics of the current candidate set. At step $t$, the state $s_t$ includes the distribution statistics of the PRM scores (mean $\mu_{\bar{R}}$, max $\max \bar{R}$, variance $\sigma^2_{\bar{R}}$), the epistemic uncertainty statistics (mean $\mu_{\sigma_R}$, max $\max \sigma_R$), the current search depth $t$, and the remaining global computation budget $T_{\text{rem}}$. Based on this state, the controller $\pi_\psi(s_t)$ outputs a comprehensive action vector $a_t \in \mathbb{R}^6$ to control the search dynamics: $a_t = [\beta_t, \gamma_t, \tau_t, \delta_t, \nu_{1,t}, \nu_{2,t}].$
These components correspond to distinct control mechanisms: the budget factors $\beta_t$ and $\gamma_t$ scale the maximum evaluation and expansion limits.
The filtering thresholds $\tau_t$ and $\delta_t$ dynamically define the boundary between ID and OOD samples and the optimism margin for retaining OOD candidates; and the temperatures $\nu_{1,t}$ and $\nu_{2,t}$ control the exploration-exploitation trade-off by determining the concentration of the re-evaluation and expansion distributions. These dynamic values substitute their static counterparts in Equation \ref{eq:eval_budget_alloc} and \ref{eq:child_alloc}. The environment provides a reward $r$ at the end of the episode to guide the learning process. We define the reward as $r = \mathbbm{1}(\hat{a} = a^*) - \lambda \cdot (\text{Cost}_{\text{used}} / \text{Cost}_{\text{total}})$, where $\mathbbm{1}(\cdot)$ is the indicator function, $\text{Cost}_{\text{used}}$ and $\text{Cost}_{\text{total}}$ are the actual and total computational cost, and $\lambda$ is a penalty coefficient.

We optimize the controller to maximize the expected reward $J(\psi) = \mathbb{E}_{\tau \sim \pi_\psi} [r(\tau)]$. To estimate the policy gradient while reducing variance, we employ the REINFORCE algorithm \cite{williams1992simple} with a batch-mean baseline $b = \frac{1}{M}\sum_{j=1}^M r(\tau_j)$. For a batch of $M$ trajectories $\{\tau_i\}_{i=1}^M$, the gradient is computed as $\nabla_\psi J(\psi) \approx \frac{1}{M} \sum_{i=1}^M \sum_{t=0}^{|\tau_i|} \nabla_\psi \log \pi_\psi(a_{i,t} | s_{i,t}) \cdot (r(\tau_i) - b)$. To address reward sparsity, a two-stage paradigm is used: first, we initialize $\psi$ via behavioral cloning to mimic the heuristic in Section \ref{sec:heuristic}; subsequently, we fine-tune the controller to learn adaptive behaviors, such as dynamically adjusting budgets and temperatures under high uncertainty.

\section{Experiments}
\label{sec:experiments}
In this section, we begin by describing the implementation details. We then present the results of the evaluation, followed by an ablation study to analyze how it works well. More results are provided in the Appendices \ref{app:full}, \ref{app:hyper} and \ref{app:abl}.

\subsection{Implementation Details}
\label{sec:implementation}
We present the implementation of our proposed H-UATS and A-UATS. The H-UATS is designed as a training-free inference strategy, requiring no parameter updates. The specific hyperparameters are detailed in Appendix \ref{app:hyper}. In order to compare against baselines under a compute-matched setting, we calibrate the relative cost of policy generation and PRM evaluation using wall-clock latency on the same A100 hardware. In our setup, a single reasoning step has an average length of 56 tokens (batch size 4). We measure that generating one step takes 571 ms on average, while a single PRM forward pass with dropout takes 32 ms on average, yielding a latency ratio of $571 / 32 \approx 17.8$. Accordingly, we treat the cost of generating one reasoning step as equivalent to 18 stochastic PRM passes under our implementation and batching configuration. All methods are then evaluated under the same budget: baselines can spend the budget on generating more candidates, whereas our method allocates an equivalent portion to parallelized uncertainty estimation.

For A-UATS, we train the budget controller using the REINFORCE algorithm on the training split of the MATH dataset. We utilize a diverse set of Policy Models: LLaMA-3.1-Instruct-8B and Qwen2.5-Instruct-3B, and PRMs: Math-Shepherd-PRM-7B and Qwen2.5-Math-PRM-7B. In each training round, we randomly sample a question and a Policy-PRM pair. The current controller then guides the generation of 10 complete reasoning trajectories. We compute the cumulative rewards based on the correctness of the final answers and estimate the advantage function to update the controller's parameters. This optimization process is conducted for a total of 500 update rounds. 

\subsection{Math Reasoning Evaluation}
\textbf{Experimental Setup} We evaluate the proposed methods on two mathematical reasoning benchmarks: MATH-500 \cite{lightmanLetsVerifyStep2023} and AIME24 \cite{ai-mo2024}. We evaluate UATS across a diverse set of frozen policy models, including LLaMA3.1-Instruct-8B, LLaMA3.2-Instruct-1B, Qwen2.5-Instruct (0.5B, 3B, and 7B). For the PRM, we include Math-Shepherd-PRM-7B, Skywork-PRM-1.5B, and Qwen2.5-Math-PRM-7B. We compare our method with the following methods: Best-of-N, Beam Search, REBASE \cite{wuInferenceScalingLaws2024}, DORA \cite{wang2025every}, and DVTS \cite{beeching2024scalingtesttimecompute}. For each method, we report answer accuracy as a function of the number of candidate paths, using $N \in \{4, 8, 16, 32, 64, 128, 256\}$. Additionally, we include an Oracle baseline (Pass@K) to represent the performance upper bound. The full results are provided in Appendix~\ref{app:full}. 

\textbf{Main Results} 
The performance of Qwen2.5-Instruct-7B, and LLaMA3.1-Instruct-8B on the MATH-500 dataset, evaluated under two PRMs: Skywork-PRM-1.5B and Qwen2.5-Math-PRM-7B on MATH-500 is shown in Figure~\ref{fig:math}. We observe that H-UATS consistently outperforms existing baselines across varying computational budgets. Furthermore, A-UATS demonstrates superior performance, surpassing H-UATS in most configurations. This confirms that dynamic allocation outperforms static heuristics by adaptively optimizing the controller based on candidate state.
 
\begin{figure}[tb]
    \centering
    \begin{subfigure}[t]{0.48\linewidth}
        \centering
        \includegraphics[width=\linewidth]{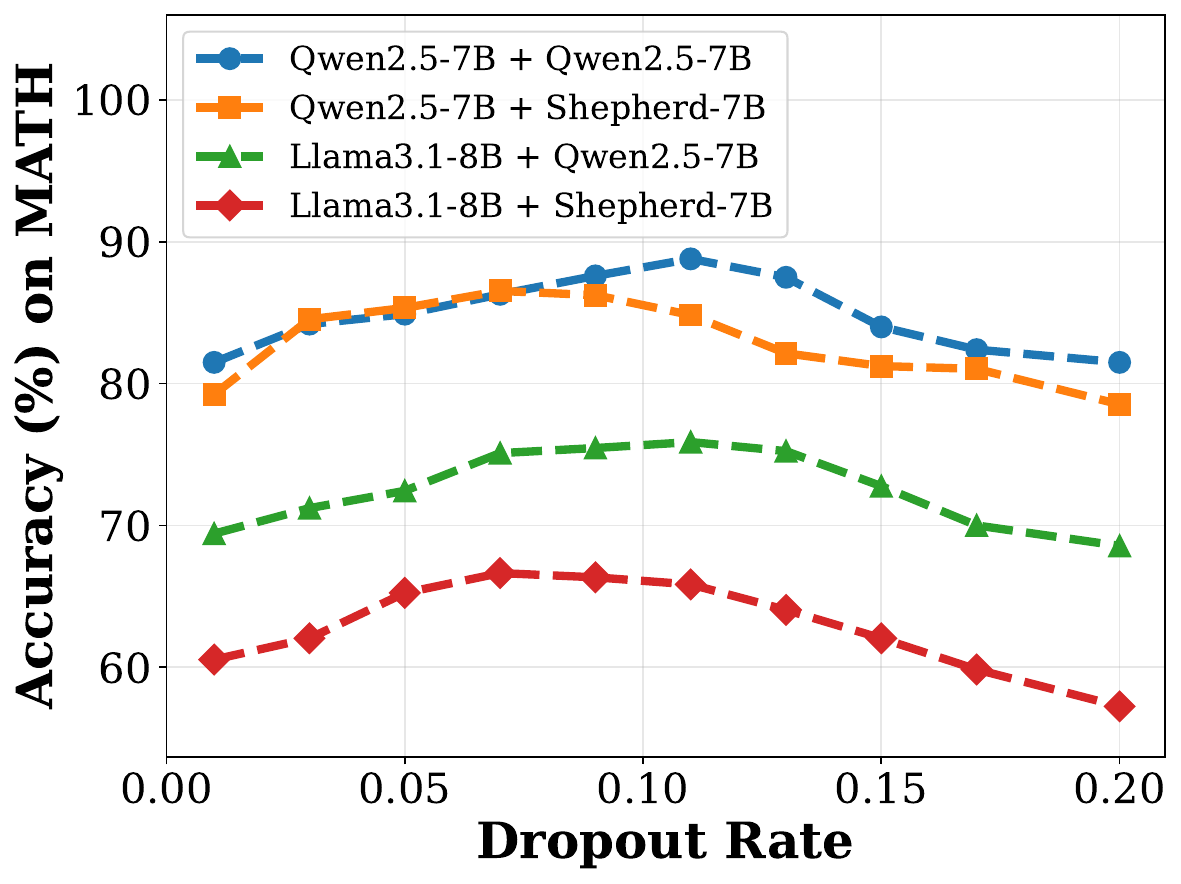}
        \caption{Ablation on dropout rate}
        \label{fig:5_a}
    \end{subfigure}
    \hfill
    \begin{subfigure}[t]{0.48\linewidth}
        \centering
    \includegraphics[width=\linewidth]{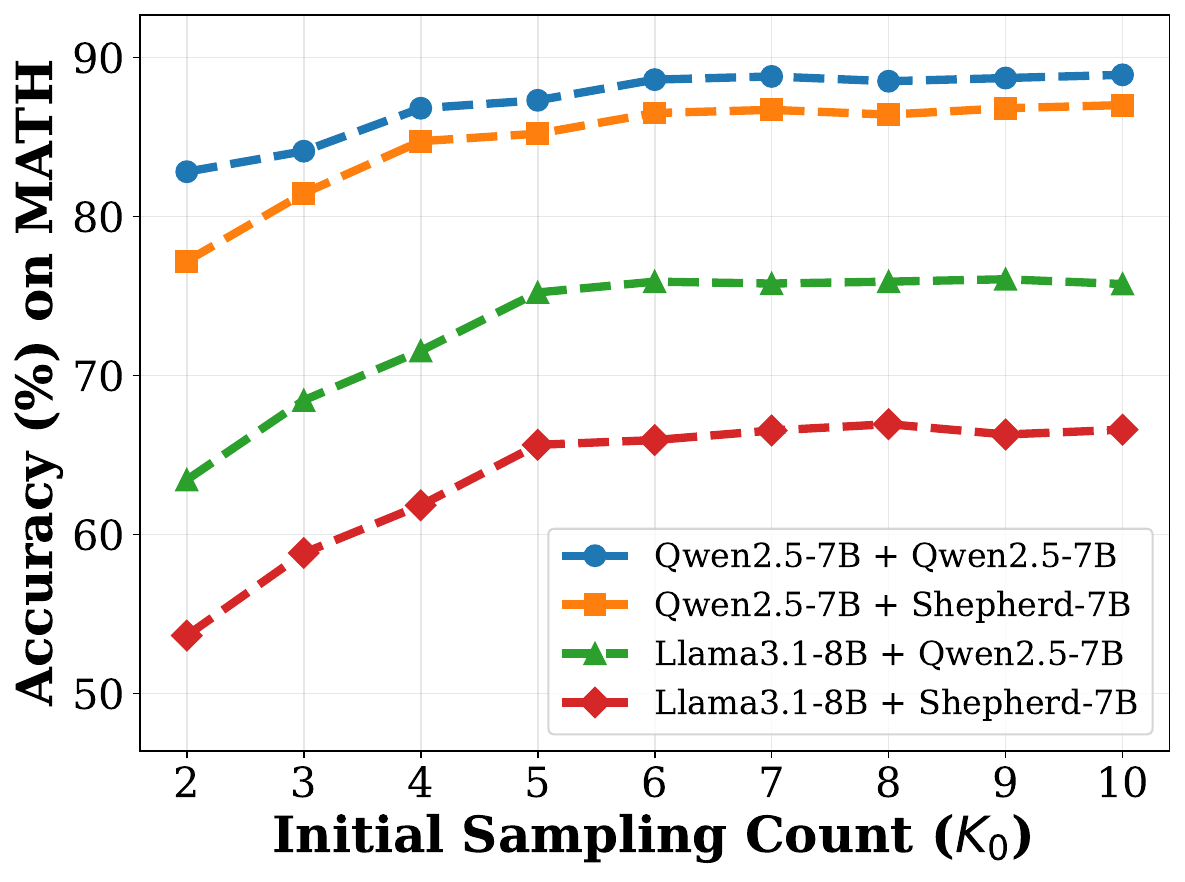}
        \caption{Ablation on $K_0$}
        \label{fig:5_b}
    \end{subfigure}
    \caption{Ablation studies on: (a) Dropout rate. (b) Initial Sampling Count for uncertainty estimation.}
    \label{fig:abl_main}
\end{figure}

\subsection{Ablation Studies}
We evaluate the effect of the dropout rate $p$ by varying it from 0.01 to 0.20 across four Policy–PRM combinations, involving the policy models Qwen2.5-Instruct-7B and LLaMA3.1-Instruct-8B, and the PRMs Qwen2.5-Math-PRM-7B and Math-Shepherd-PRM-7B. This parameter controls the stochasticity of the uncertainty estimates. As illustrated in Figure~\ref{fig:5_a}, insufficient noise with a lower $p$ fails to capture epistemic uncertainty, while excessive noise with higher $p$ destabilizes the reward signal. Notably, the optimal rate depends on the verifier's architecture; Qwen2.5-Math-PRM peaks at $p \approx 0.10$, whereas Math-Shepherd-PRM requires a more conservative $p \approx 0.07$, reflecting different tolerance levels to feature perturbations.

We evaluate the initial sampling count $K_0 \in [2, 10]$, which establishes the baseline uncertainty for candidate filtering. As shown in Figure~\ref{fig:5_b}, accuracy improves consistently with higher $K_0$, as more samples reduce the variance of the initial estimators. However, the marginal gains diminish rapidly, with the performance curve plateauing around $K_0 = 7$. This suggests that a modest number of initial passes is sufficient to distinguish ID from OOD candidates.

\section{Conclusion}
In this work, we identify the critical challenge of epistemic uncertainty in PRMs, demonstrating that distribution shifts during inference undermine search reliability. We provide theoretical analysis to prove that uncertainty-aware strategies can reduce search regret from linear to sublinear growth. Building on these insights, we propose UATS, a framework that dynamically optimizes compute allocation via a learned controller. Extensive experiments confirm that our approach effectively mitigates OOD errors, achieving superior accuracy compared to baselines. Our findings underscore the vital role of uncertainty calibration in reliable reasoning.

\section*{Impact Statement}
This paper presents work whose goal is to advance the field of Machine Learning. There are many potential societal consequences of our work, none which we feel must be specifically highlighted here.

\bibliography{main}
\bibliographystyle{icml2026}

\newpage
\appendix
\onecolumn

\section*{Appendix}

The appendix is organized as follows:
\begin{enumerate}
    \item \textbf{Appendix \ref{app:proof_1}} presents the proof of Proposition 4.1.
    \item \textbf{Appendix \ref{app:proof_2}} presents the proof of Proposition 4.2.
    \item \textbf{Appendix \ref{app:compare}} provides comparison and discussion with related works.
    \item \textbf{Appendix \ref{app:hyper}} provides specific implementation of Hyperparameters.
    \item \textbf{Appendix \ref{app:full}} provides the full results.
    \item \textbf{Appendix \ref{app:abl}} provides additional ablation studies.
    \item \textbf{Appendix \ref{app:quali}} provides qualitative results.
\end{enumerate}

\section{Proof of Proposition \ref{prop:linear_degrade_xi_compact}}
\label{app:proof_1}
\begin{proof}
We analyze the degradation of the expected accuracy by examining the error probability at each reasoning step $t \in \{1, \dots, T\}$.

Let $E_t$ denote the event that the selected hypothesis is incorrect at step $t$, i.e., $E_t = \{ \hat h_t \neq h_t^* \}$. We aim to lower-bound the probability of this error event, $\mathbb{P}(E_t)$. Using the Law of Total Probability with respect to the OOD event $\mathcal{O}_t$, we have:
\begin{equation}
    \mathbb{P}(E_t) = \mathbb{P}(E_t \mid \mathcal{O}_t)\mathbb{P}(\mathcal{O}_t) + \mathbb{P}(E_t \mid \mathcal{O}_t^c)\mathbb{P}(\mathcal{O}_t^c).
\end{equation}
Applying the assumptions given in the proposition:
\begin{enumerate}
    \item[(i)] The OOD probability is fixed: $\mathbb{P}(\mathcal{O}_t) = \varepsilon$.
    \item[(iii)] The conditional error probabilities satisfy $\mathbb{P}(E_t \mid \mathcal{O}_t) \ge \rho$ and $\mathbb{P}(E_t \mid \mathcal{O}_t^c) = 0$.
\end{enumerate}
Substituting these into the equation, we obtain a lower bound for the error probability at step $t$:
\begin{equation}
    \mathbb{P}(E_t) \ge \rho \cdot \varepsilon + 0 \cdot (1-\varepsilon) = \varepsilon\rho.
    \label{eq:prob_error_bound}
\end{equation}

Next, we quantify the impact of these errors on the final accuracy. Let $L_t$ be the random variable representing the reduction in success probability caused by the selection at step $t$. According to assumption (ii), if an error occurs ($E_t$ holds), the success probability decreases by at least $\underline{\Delta}$; otherwise, the reduction is non-negative (we assume no improvement from incorrect steps). Thus, we can bound the expected loss at step $t$ as:
\begin{equation}
    \mathbb{E}[L_t] \ge \underline{\Delta} \cdot \mathbb{P}(E_t) + 0 \cdot \mathbb{P}(E_t^c).
\end{equation}
Using the bound from Eq.~\eqref{eq:prob_error_bound}:
\begin{equation}
    \mathbb{E}[L_t] \ge \underline{\Delta} \cdot (\varepsilon\rho) = \varepsilon\rho\underline{\Delta}.
\end{equation}

The final accuracy $\mathrm{Acc}_T$ is the initial potential reward $R^*(h_0)$ minus the cumulative losses over all $T$ steps. Taking the expectation:
\begin{equation}
    \mathbb{E}[\mathrm{Acc}_T] \le R^*(h_0) - \sum_{t=1}^{T} \mathbb{E}[L_t].
\end{equation}
Substituting the lower bound for the expected loss at each step:
\begin{align}
    \mathbb{E}[\mathrm{Acc}_T] &\le R^*(h_0) - \sum_{t=1}^{T} \varepsilon\rho\underline{\Delta} \nonumber \\
    &= R^*(h_0) - T\,\varepsilon\,\underline{\Delta}\,\rho.
\end{align}
This concludes the proof.
\end{proof}

\section{Proof of Proposition \ref{prop:sublinear_degrade_ucb}}
\label{app:proof_2}
\begin{proof}
To analyze the expected performance, we define the instantaneous regret at step $t$ as $r_t \triangleq R^*(h_t^*) - R^*(\hat h_t)$. The final accuracy is the initial potential minus the cumulative regret: $\mathbb{E}[\mathrm{Acc}_T] = R^*(h_0) - \sum_{t=1}^T \mathbb{E}[r_t]$.

First, we analyze the concentration of the reward estimates. Let $\mathcal{E}_t$ denote the ``good event'' where the empirical mean estimates for all candidates in $\mathcal{H}_t$ are close to their true values within the confidence interval $U_t$:
\begin{equation}
    \mathcal{E}_t = \left\{ \forall h \in \mathcal{H}_t, \ | \bar{R}_t(h) - R^*(h) | \le U_t \right\},
\end{equation}
where $U_t = \sqrt{\frac{2\ln t}{K_t}}$.
Using Hoeffding's inequality for a single candidate $h$, the probability of deviation is bounded by:
\begin{equation}
    \mathbb{P}(|\bar{R}_t(h) - R^*(h)| > U_t) \le 2\exp\left(-2 K_t U_t^2\right) = 2\exp(-4\ln t) = 2t^{-4}.
\end{equation}
Applying the union bound over all $M$ candidates in $\mathcal{H}_t$, the probability that the good event fails is:
\begin{equation}
    \mathbb{P}(\mathcal{E}_t^c) \le \sum_{h \in \mathcal{H}_t} 2t^{-4} = 2M t^{-4}.
    \label{eq:prob_bad_event}
\end{equation}

Next, we bound the expected regret $\mathbb{E}[r_t]$. We decompose this expectation based on the occurrence of the OOD event $\mathcal{O}_t$. Recall that $\mathbb{P}(\mathcal{O}_t)=\varepsilon$.
\begin{enumerate}
    \item \textbf{Case 1: No OOD ($\mathcal{O}_t^c$).} We assume the model performs optimally or near-optimally in distribution, contributing negligible regret compared to the OOD case.
    \item \textbf{Case 2: OOD occurs ($\mathcal{O}_t$).} We further decompose based on the concentration event $\mathcal{E}_t$:
    \begin{itemize}
        \item If $\mathcal{E}_t$ holds: The UCB selection strategy guarantees a regret bound. By definition of $\hat h_t$ (optimizing the upper bound) and $\mathcal{E}_t$ (true value contained in bound):
        \begin{equation}
            R^*(h_t^*) \le \bar{R}_t(h_t^*) + U_t \le \bar{R}_t(\hat h_t) + U_t \le R^*(\hat h_t) + 2U_t.
        \end{equation}
        Thus, $r_t = R^*(h_t^*) - R^*(\hat h_t) \le 2U_t$.
        \item If $\mathcal{E}_t^c$ holds: The regret is bounded by the maximum possible reward difference, denoted by $C_{\max}$ (e.g., 1). This occurs with probability at most $2Mt^{-4}$.
    \end{itemize}
\end{enumerate}

Combining these, the expected regret at step $t$ is dominated by the OOD interaction:
\begin{align}
    \mathbb{E}[r_t] &\le \varepsilon \left( \mathbb{P}(\mathcal{E}_t) \cdot 2U_t + \mathbb{P}(\mathcal{E}_t^c) \cdot C_{\max} \right) \nonumber \\
    &\le 2\varepsilon U_t + 2M C_{\max} \varepsilon t^{-4}.
\end{align}

Finally, we sum the expected regret over $t=1$ to $T$ to bound the total degradation:
\begin{equation}
    R^*(h_0) - \mathbb{E}[\mathrm{Acc}_T] = \sum_{t=1}^T \mathbb{E}[r_t] \le \sum_{t=1}^T \left( 2\varepsilon U_t + O(\varepsilon t^{-4}) \right).
\end{equation}
Since $\sum_{t=1}^\infty t^{-4}$ converges to a constant, the second term is $O(1)$. For the first term, substituting $U_t = \sqrt{\frac{2\ln t}{K_t}}$:
\begin{equation}
    \text{Total Regret} \le 2\varepsilon \sum_{t=1}^T \sqrt{\frac{2\ln t}{K_t}} + O(1).
\end{equation}
If $K_t = \Omega(t)$, then $U_t \approx \sqrt{\frac{\ln t}{t}}$. Approximating the sum by an integral $\int_1^T \sqrt{\frac{\ln x}{x}} dx \approx \sqrt{T \ln T}$, we obtain:
\begin{equation}
    R^*(h_0) - \mathbb{E}[\mathrm{Acc}_T] = O\left( \varepsilon \sqrt{T \ln T} \right).
\end{equation}
This confirms the sublinear degradation.
\end{proof}

\section{Comparison with Related Works}
\label{app:compare}

\subsection{Comparison with REBASE}
REBASE \cite{wuInferenceScalingLaws2024} allocates inference compute based on a score-proportional principle, implicitly treating PRM scores as reliable point estimates. However, this assumption fails under distribution shifts, where PRMs often exhibit high confidence on incorrect OOD steps. UATS fundamentally differs by introducing an epistemic uncertainty gate: instead of blindly trusting raw scores, we explicitly model the reliability of the verifier. While REBASE asks ``where to expand based on quality,'' UATS first answers ``whether the quality assessment is trustworthy.'' By selectively re-evaluating high-variance candidates, UATS prevents the over-expansion of hallucinated high-scoring paths, a critical failure mode that REBASE cannot address. Furthermore, A-UATS extends this by dynamically learning the optimal trade-off between verification (re-evaluation) and search breadth, whereas REBASE relies on static hyperparameters.

\subsection{Comparison with DORA}
DORA \cite{wang2025every} focuses on macro-level resource allocation, aiming to reduce redundancy by allocating rollout budgets to distinct reasoning directions (i.e., solution clusters). In contrast, UATS addresses micro-level verification reliability, targeting the epistemic errors of the PRM at individual reasoning steps. The two methods solve orthogonal inefficiencies: DORA mitigates the diminishing returns of sampling similar paths, while UATS mitigates the risk of being misled by OOD generalizations. Consequently, they are highly complementary; UATS provides calibrated step-level signals that can enhance the direction-level decision-making in DORA.

\subsection{Comparison with MCTS}
Monte Carlo Tree Search (MCTS) estimates node values via computationally expensive, sequential lookahead rollouts. In LLM reasoning, such deep rollouts incur prohibitive latency and compute costs \cite{snellScalingLLMTestTime2024a}. UATS offers a more efficient alternative for value estimation: rather than generating future tokens (rollouts), we estimate value uncertainty via lightweight, parallelizable stochastic forward passes (MC Dropout) of the PRM. This design allows UATS to correct valuation errors and prune dead ends without the heavy overhead of sequential generation, making it significantly more scalable for test-time compute allocation.

\section{Implementation of Hyperparameters}
\label{app:hyper}

We provide a detailed listing of the hyperparameters used for both the Heuristic Uncertainty-Aware Tree Search (H-UATS) and the Adaptive Uncertainty-Aware Tree Search (A-UATS) in Table~\ref{tab:hyperparameters}. 

\begin{table}[h]
    \centering
    \caption{Hyperparameter settings for H-UATS and A-UATS experiments.}
    \label{tab:hyperparameters}
    \begin{tabular}{llc}
        \toprule
        \textbf{Method} & \textbf{Hyperparameter} & \textbf{Value} \\
        \midrule
        \multirow{3}{*}{Common} & Dropout Rate ($p$) & 0.07 $\sim$ 0.10 \\
                                & Max Candidate Paths ($N$) & 256 \\
                                & Exploration parameter ($\alpha$) & 0.3 \\
        \midrule
        \multirow{5}{*}{H-UATS} & Initial Sampling Count ($K_0$) & 7 \\
                                & Uncertainty Threshold ($\tau$) & 0.003 \\
                                & Optimism Margin ($\delta$) & 0.04 \\
                                & Re-evaluation Temperature ($\nu_1$) & 0.5 \\
                                & Expansion Temperature ($\nu_2$) & 0.2 \\
        \midrule
        \multirow{3}{*}{A-UATS} & Learning Rate & $1 \times 10^{-4}$ \\
                                & Optimizer & AdamW \\
                                & Cost Penalty ($\lambda$) & 0.05 \\
        \bottomrule
    \end{tabular}
\end{table}

\section{Full Results}
\label{app:full}
The full results of all policy models and PRMs in the MATH-500 dataset are provided in Figure \ref{fig:full_MATH}, and the results of AIME are provided in Figure \ref{fig:full_AIME}.

\begin{figure}[ht]
    \centering
    \includegraphics[width=\linewidth]{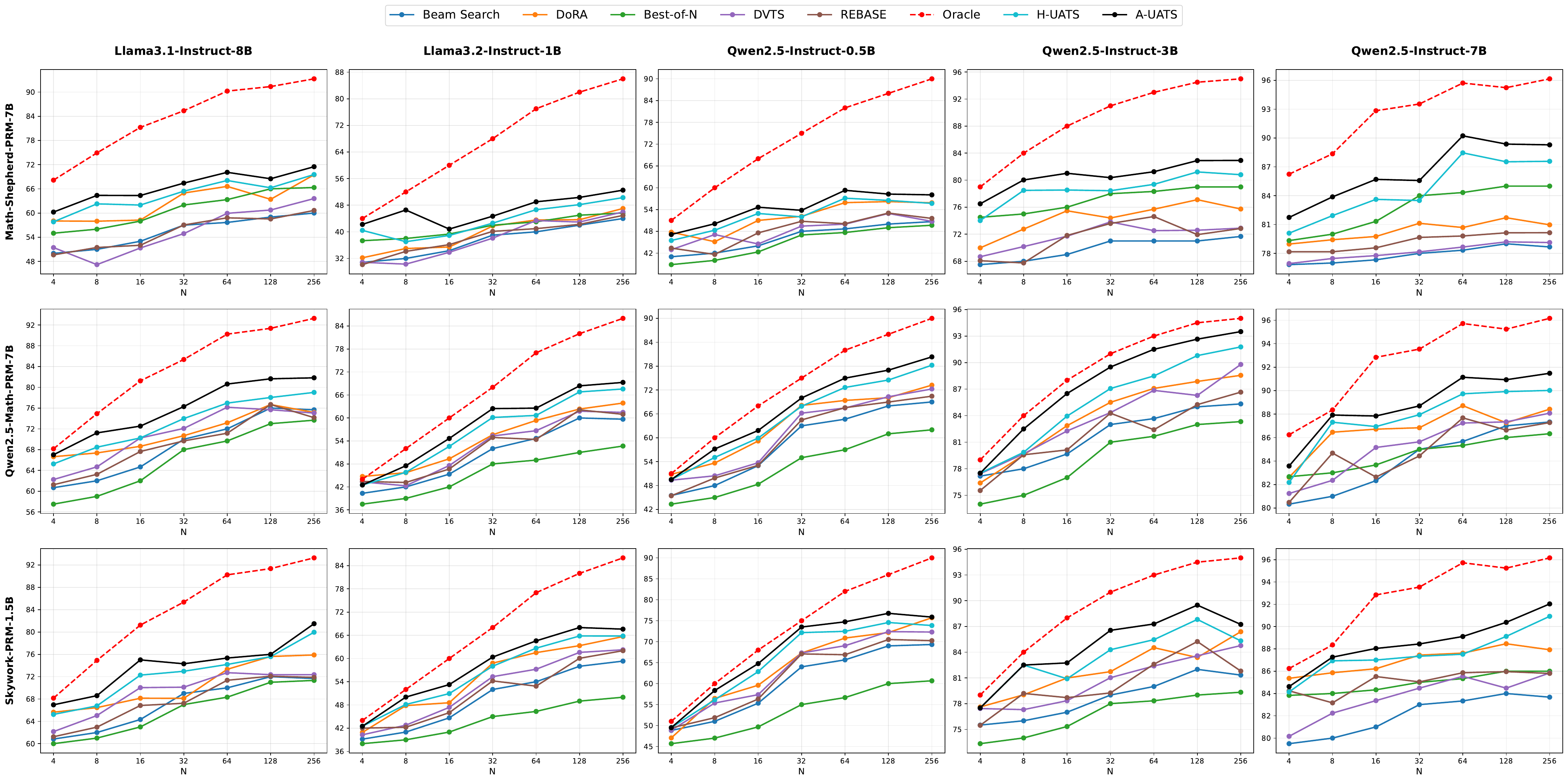}
    \caption{Full results on the MATH500 dataset.}
    \label{fig:full_MATH}
\end{figure}

\begin{figure}[ht]
    \centering
    \includegraphics[width=\linewidth]{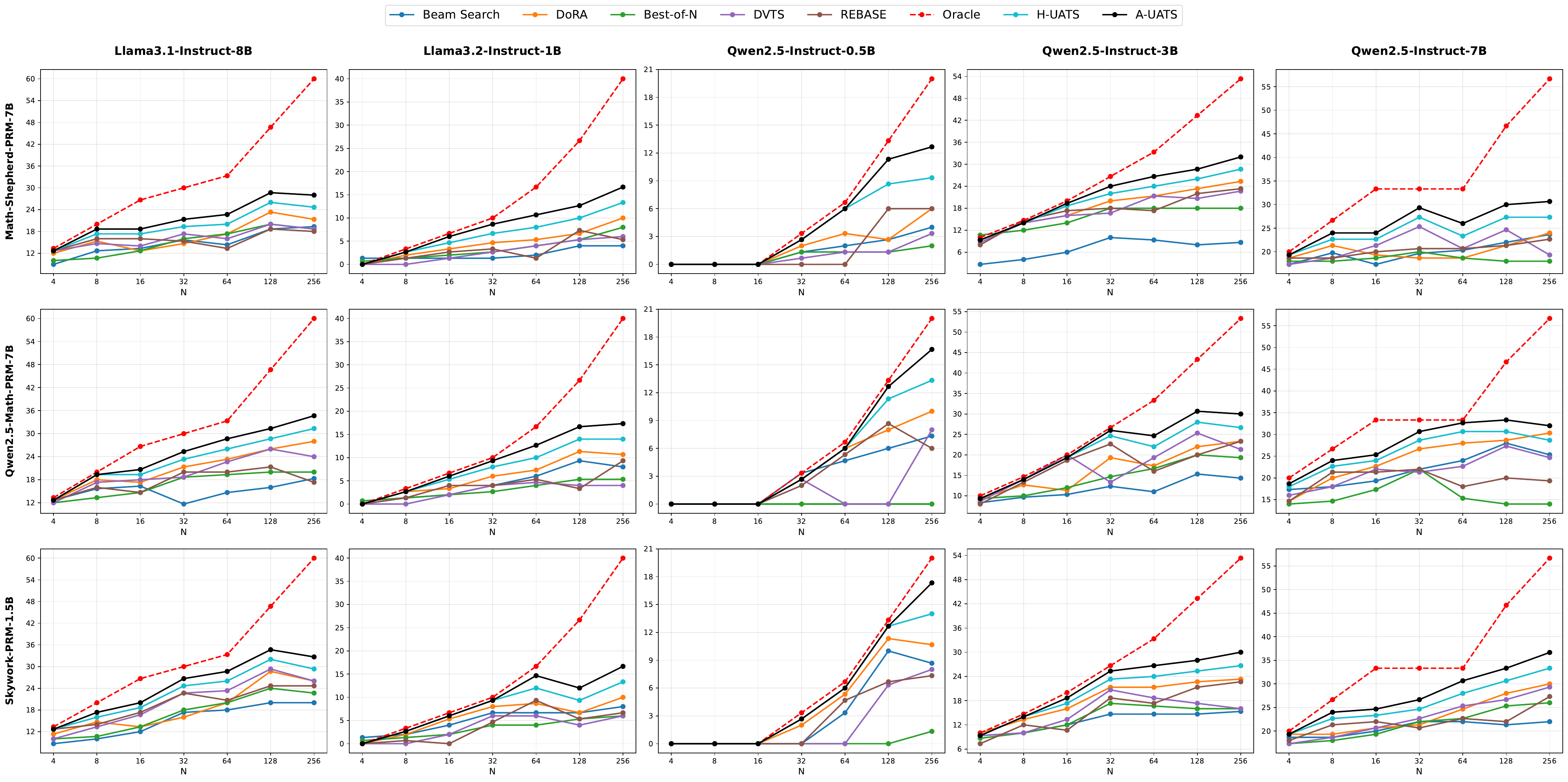}
    \caption{Full results on the AIME24 dataset.}
    \label{fig:full_AIME}
\end{figure}

\begin{table*}[h]
\centering
\caption{Average performance across all number of $N$ for different combinations on AIME and MATH datasets.}
\label{tab:results_summary}
\resizebox{\textwidth}{!}{
\begin{tabular}{cc|ccccccc|ccccccc}
\toprule
\multicolumn{2}{c|}{\textbf{Model Pair}} & \multicolumn{7}{c|}{\textbf{AIME}} & \multicolumn{7}{c}{\textbf{MATH}} \\
\textbf{Policy} & \textbf{Reward} & \textbf{Beam Search} & \textbf{Best-of-N} & \textbf{DoRA} & \textbf{DVTS} & \textbf{REBASE} & \textbf{H-UATS} & \textbf{A-UATS} & \textbf{Beam Search} & \textbf{Best-of-N} & \textbf{DoRA} & \textbf{DVTS} & \textbf{REBASE} & \textbf{H-UATS} & \textbf{A-UATS} \\
\midrule
Llama3.1-Instruct-8B & Math-Shepherd-PRM-7B & 14.71 & 15.05 & 16.67 & 16.19 & 15.71 & 19.62 & \textbf{21.52} & 55.38 & 60.95 & 62.67 & 55.59 & 55.45 & 64.48 & \textbf{66.64} \\
Llama3.1-Instruct-8B & Qwen2.5-Math-PRM-7B & 15.05 & 16.86 & 20.95 & 19.81 & 17.43 & 22.95 & \textbf{24.67} & 68.72 & 66.12 & 71.19 & 70.87 & 69.11 & 73.14 & \textbf{75.88} \\
Llama3.1-Instruct-8B & Skywork-PRM-1.5B & 15.14 & 16.95 & 18.57 & 20.19 & 19.53 & 22.76 & \textbf{24.67} & 67.12 & 65.95 & 70.46 & 69.26 & 67.66 & 72.43 & \textbf{73.96} \\
Llama3.2-Instruct-1B & Math-Shepherd-PRM-7B & 2.19 & 3.43 & 4.57 & 2.76 & 3.05 & 6.48 & \textbf{8.19} & 37.45 & 41.48 & 39.58 & 37.52 & 38.40 & 44.15 & \textbf{46.80} \\
Llama3.2-Instruct-1B & Qwen2.5-Math-PRM-7B & 4.48 & 3.05 & 5.90 & 2.67 & 3.90 & 7.71 & \textbf{9.24} & 50.57 & 45.60 & 55.12 & 52.81 & 51.77 & 56.57 & \textbf{58.19} \\
Llama3.2-Instruct-1B & Skywork-PRM-1.5B & 5.05 & 3.33 & 5.81 & 3.43 & 3.81 & 7.52 & \textbf{8.76} & 49.74 & 44.05 & 55.17 & 52.47 & 50.83 & 56.27 & \textbf{58.05} \\
Qwen2.5-Instruct-0.5B & Math-Shepherd-PRM-7B & 1.43 & 0.86 & 2.00 & 0.95 & 1.71 & 3.81 & \textbf{4.67} & 46.33 & 44.93 & 51.30 & 48.87 & 47.82 & 52.57 & \textbf{54.46} \\
Qwen2.5-Instruct-0.5B & Qwen2.5-Math-PRM-7B & 3.05 & 0.00 & 3.81 & 1.52 & 3.14 & 4.76 & \textbf{5.43} & 58.74 & 53.09 & 63.29 & 61.77 & 60.04 & 65.40 & \textbf{67.26} \\
Qwen2.5-Instruct-0.5B & Skywork-PRM-1.5B & 3.14 & 0.19 & 4.19 & 2.05 & 2.67 & 5.05 & \textbf{5.52} & 60.45 & 53.53 & 65.09 & 63.71 & 61.52 & 65.94 & \textbf{67.64} \\
Qwen2.5-Instruct-3B & Math-Shepherd-PRM-7B & 6.95 & 15.52 & 18.47 & 17.14 & 17.24 & 20.38 & \textbf{22.00} & 69.88 & 77.12 & 74.59 & 72.19 & 71.25 & 78.70 & \textbf{80.71} \\
Qwen2.5-Instruct-3B & Qwen2.5-Math-PRM-7B & 11.62 & 14.57 & 16.47 & 17.43 & 17.43 & 20.57 & \textbf{22.00} & 81.69 & 79.29 & 86.33 & 84.27 & 82.79 & 85.63 & \textbf{87.66} \\
Qwen2.5-Instruct-3B & Skywork-PRM-1.5B & 12.95 & 13.81 & 18.28 & 15.05 & 15.71 & 20.00 & \textbf{21.71} & 78.69 & 76.76 & 83.39 & 81.56 & 79.96 & 83.40 & \textbf{84.76} \\
Qwen2.5-Instruct-7B & Math-Shepherd-PRM-7B & 20.02 & 18.48 & 20.29 & 21.05 & 20.38 & 24.29 & \textbf{26.19} & 77.88 & 82.71 & 80.38 & 78.19 & 79.24 & 84.67 & \textbf{86.54} \\
Qwen2.5-Instruct-7B & Qwen2.5-Math-PRM-7B & 22.00 & 15.90 & 24.43 & 21.71 & 19.52 & 26.19 & \textbf{28.10} & 84.09 & 84.57 & 86.72 & 85.29 & 84.83 & 87.72 & \textbf{88.80} \\
Qwen2.5-Instruct-7B & Skywork-PRM-1.5B & 20.67 & 21.52 & 23.33 & 22.95 & 22.00 & 26.00 & \textbf{27.90} & 82.07 & 84.93 & 86.98 & 83.73 & 85.08 & 87.57 & \textbf{88.55} \\
\bottomrule
\end{tabular}
}
\end{table*}

Table~\ref{tab:results_summary} summarizes the average performance across all 15 model combinations (averaged over all budget levels $N \in \{4, \dots, 256\}$). On the challenging AIME benchmark, H-UATS achieves an average accuracy of 15.87\%, surpassing the strongest baseline (DoRA, 13.58\%) by a margin of 2.29\%. A-UATS further elevates this performance to 17.37\%, representing a net improvement of 3.79\% over DoRA and 1.50\% over H-UATS. A similar trend is observed on the MATH-500 dataset, where H-UATS (70.58\%) and A-UATS (73.73\%) consistently outperform the best baseline (DoRA, 68.81\%). Specifically, A-UATS demonstrates a substantial 4.92\% average gain over DoRA. These aggregated results strongly validate our core hypothesis: explicitly modeling uncertainty (H-UATS) provides a robust advantage over point-estimate methods, while learning a dynamic allocation policy (A-UATS) further maximizes the utility of the available compute budget across diverse reasoning scenarios.

\section{Additional Ablation Studies}
\label{app:abl}
\subsection{The Role of Uncertainty Estimation in H-UATS and A-UATS}

To investigate the impact of Uncertainty Estimation in the UATS framework, we conduct a comprehensive ablation study comparing model performance with and without the uncertainty estimation mechanism. To ensure the robustness of our findings, we evaluate this impact across diverse combinations of Policy Models and Process Reward Models. Specifically, we test combinations involving Qwen2.5-Instruct-7B and LLaMA3.1-Instruct-8B as policies, paired with Qwen2.5-Math-PRM-7B and Math-Shepherd-PRM-7B. We introduce two ablated baselines:
\begin{itemize}
    \item \textbf{H-UATS w/o Uncert.}: This variant removes the Monte Carlo Dropout mechanism from the heuristic search. Instead of using uncertainty-guided re-evaluation, it relies solely on the single-pass PRM score for budget allocation. This setup is methodologically similar to REBASE \cite{wuInferenceScalingLaws2024}, but differs in specific hyperparameter settings.
    \item \textbf{A-UATS w/o Uncert.}: This variant trains the adaptive controller without access to uncertainty-related state features (i.e., $\mu_{\sigma_R}$ and $\max \sigma_R$ are removed from the state space $s_t$). The controller must make budget allocation decisions based solely on the distribution of raw PRM scores and the remaining budget.
\end{itemize}

\begin{table}[ht]
    \centering
    \caption{Ablation study evaluating the impact of Uncertainty Estimation across different Policy and PRM combinations. We report accuracy (\%) on MATH-500 and AIME24.}
    \label{tab:abl_uncert}
    \begin{tabular}{cclcc}
        \toprule
        \textbf{Policy Model} & \textbf{PRM} & \textbf{Method} & \textbf{MATH-500} & \textbf{AIME24} \\
        \midrule
        \multirow{4}{*}{Qwen2.5-7B-Instruct} & \multirow{4}{*}{Qwen2.5-Math-PRM-7B} 
             & H-UATS w/o Uncert. & 85.62 & 18.82 \\
             & & H-UATS             & \textbf{87.72} & \textbf{26.19} \\
             & & A-UATS w/o Uncert. & 85.52 & 22.47 \\
             & & A-UATS             & \textbf{88.80} & \textbf{28.10} \\
        \midrule
        \multirow{4}{*}{Qwen2.5-7B-Instruct} & \multirow{4}{*}{Math-Shepherd-PRM-7B} 
             & H-UATS w/o Uncert. & 80.42 & 21.46 \\
             & & H-UATS             & \textbf{84.67} & \textbf{24.29} \\
             & & A-UATS w/o Uncert. & 82.16 & 21.42 \\
             & & A-UATS             & \textbf{86.54} & \textbf{26.19} \\
        \midrule
        \multirow{4}{*}{LLaMA3.1-8B-Instruct} & \multirow{4}{*}{Math-Shepherd-PRM-7B} 
             & H-UATS w/o Uncert. & 57.72 & 16.01 \\
             & & H-UATS             & \textbf{64.48} & \textbf{19.62} \\
             & & A-UATS w/o Uncert. & 61.02 & 16.67 \\
             & & A-UATS             & \textbf{66.64} & \textbf{21.52} \\
        \bottomrule
    \end{tabular}
\end{table}

The results in Table \ref{tab:abl_uncert} consistently demonstrate the critical role of uncertainty estimation. In all three settings, both H-UATS and A-UATS significantly outperform their uncertainty-agnostic counterparts. For instance, in the heterogeneous setting (LLaMA3.1-8B + Math-Shepherd), incorporating uncertainty into H-UATS yields a substantial gain of +6.76\% on MATH-500. This confirms our hypothesis that uncertainty estimates act as a vital filter for OOD errors.

Notably, A-UATS w/o Uncert. generally outperforms H-UATS w/o Uncert. even without explicit uncertainty features. This suggests that the learned controller can implicitly infer some notion of "difficulty" or "reliability" from the distribution of raw scores alone and optimize the budget allocation better than fixed heuristic rules. However, the full A-UATS method achieves the best performance across the board, indicating that explicit epistemic uncertainty provides a unique and irreplaceable signal for robust decision-making that cannot be fully recovered from point-estimate scores alone.

\begin{figure}[tb]
    \centering
    \begin{subfigure}[t]{0.48\linewidth}
        \centering
        \includegraphics[width=\linewidth]{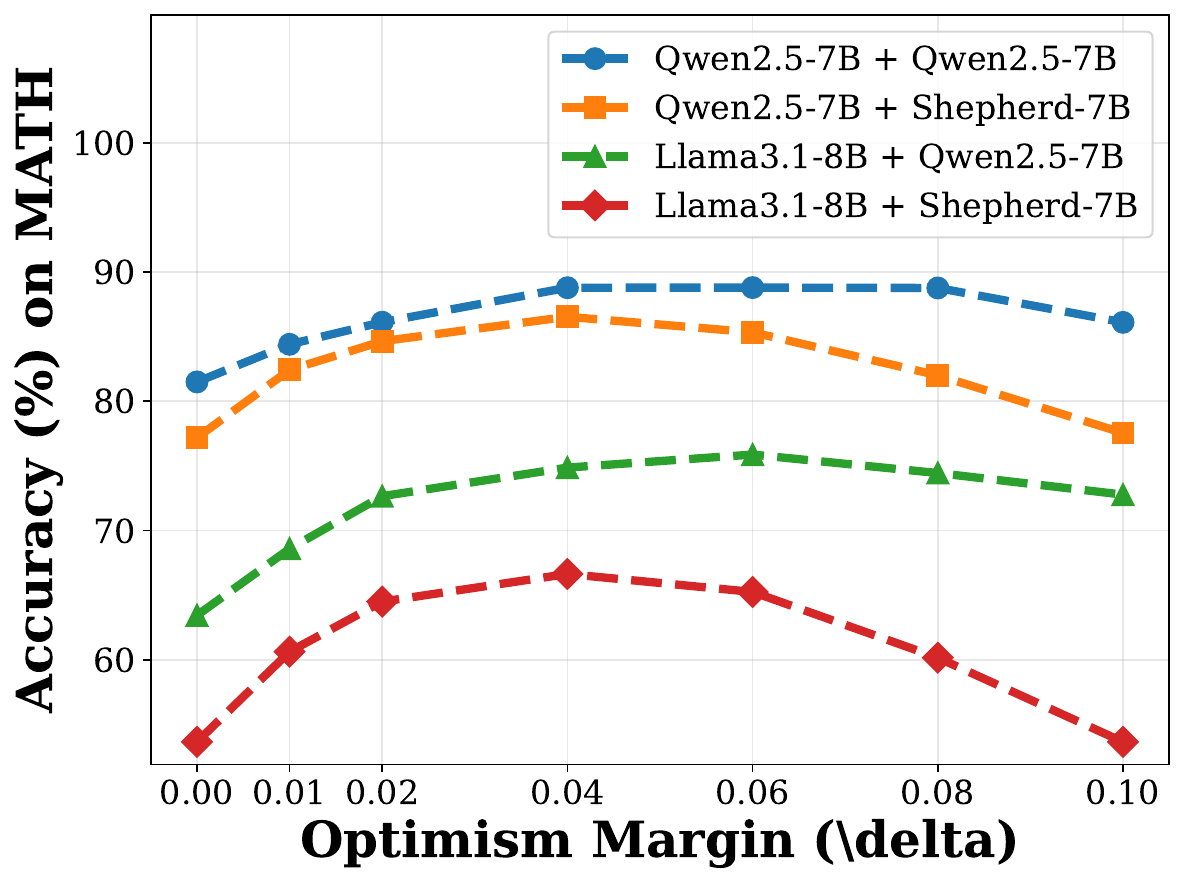}
        \caption{Optimism Margin $\delta$}
        \label{fig:ablation_delta}
    \end{subfigure}
    \hfill
    \begin{subfigure}[t]{0.48\linewidth}
        \centering
    \includegraphics[width=\linewidth]{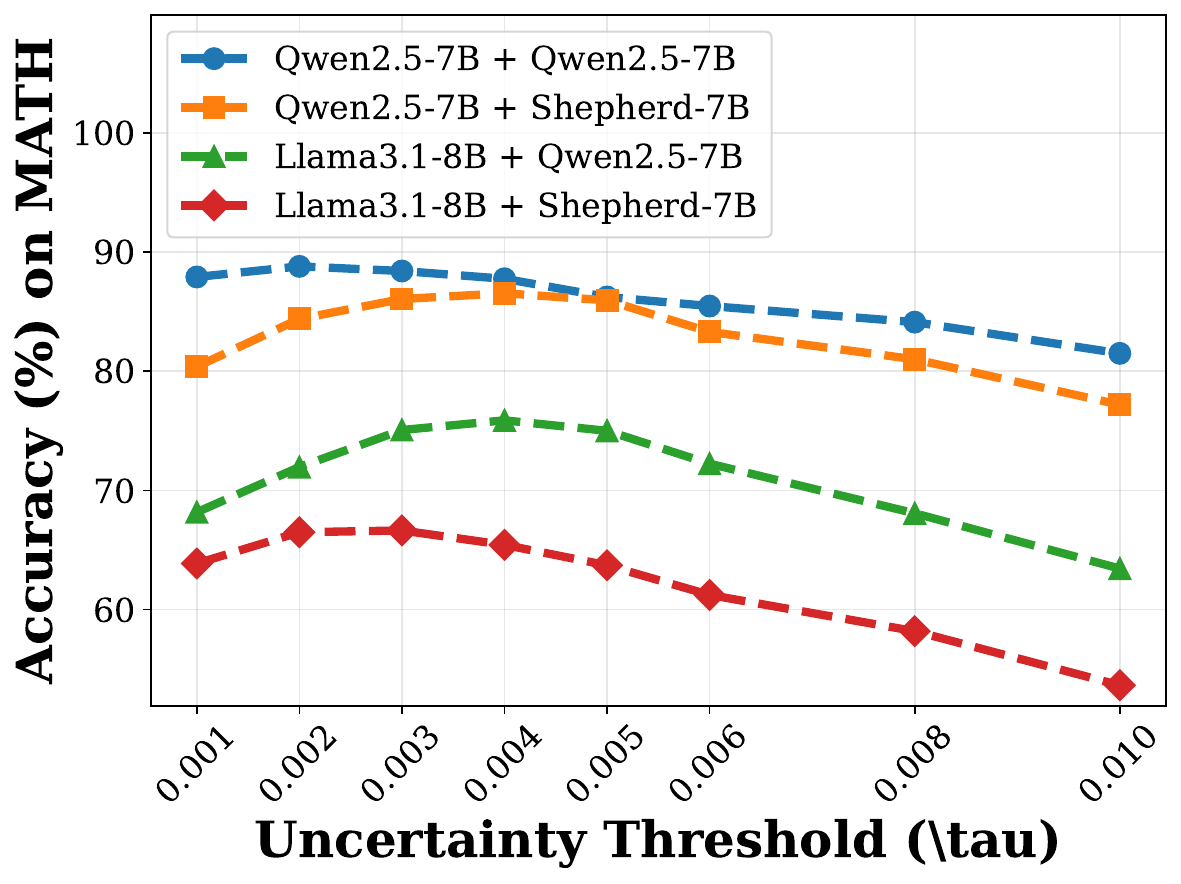}
        \caption{Uncertainty Threshold $\tau$}
        \label{fig:ablation_tau}
    \end{subfigure}
    \caption{Ablation studies for $\tau$ and $\delta$.}
    \label{fig:abl_app}
\end{figure}

\subsection{Ablation of Uncertainty Threshold}
The uncertainty threshold $\tau$ plays a pivotal role in distinguishing between ID and OOD reasoning steps. It acts as a gatekeeper, determining which candidates warrant expensive re-evaluation. We analyze the sensitivity of H-UATS performance to varying $\tau$. Empirically, setting $\tau$ too low results in an overly aggressive classification, where even stable ID steps are flagged as uncertain, causing the system to waste its limited evaluation budget on unnecessary verifications. Conversely, a high $\tau$ fails to capture genuine epistemic uncertainties, allowing hallucinated or high-error OOD steps to propagate through the search tree. Our results, visualized in Figure \ref{fig:ablation_tau}, indicate that an optimal trade-off is achieved around $\tau \approx 0.003$, where the system effectively filters high-risk nodes while maintaining search efficiency.

\subsection{Ablation of Optimism Margin}
The optimism margin $\delta$ defines the degree of leniency extended to potentially underestimated candidates during the filtering phase. It allows the system to rescue steps that have low initial PRM scores but high uncertainty. Increasing $\delta$ encourages exploration by retaining more uncertain candidates for re-evaluation. Figure \ref{fig:ablation_delta} shows that excessive optimism (large $\delta$) introduces significant noise, as the system begins to invest resources in truly incorrect steps. The performance curve suggests a optimal value when $\delta \approx 0.04$.

\subsection{Ablation of Network Structure}
In this section, we investigate how different architectural choices for the policy networks affect the performance of A-UATS. We perform an ablation study on the MATH-500 dataset using Qwen2.5-Instruct-7B as the policy model and Math-Shepherd-PRM-7B as PRM. For each configuration, we report the average accuracy across all reward models and compute budgets. Specifically, we vary the number of layers and hidden dimensions of both networks to assess their impact on overall performance.
\begin{table}[ht]
\centering
\caption{Mean accuracy of A-UATS on MATH-500 under different network architectures. Each result is averaged over all compute budgets.}
\label{tab:actor-critic-ablation}
\begin{tabular}{c|c|c}
\toprule
\textbf{Network Layers} & \textbf{Network Dim} & \textbf{Accuracy (\%)} \\
\midrule
2 & 128 & 84.73 \\
2 & 256 &\bf 86.54 \\
2 & 512 & 85.12 \\
3 & 128 & 86.32 \\
3 & 256 & 85.61 \\
3 & 512 & 85.03 \\
\bottomrule
\end{tabular}
\end{table}

\section{Qualitative Results}
\label{app:quali}
To provide deeper insights into the behavior of the PRM and the utility of epistemic uncertainty, we analyze four representative reasoning steps encountered during the standard beam search process. These examples, visualized in Figures \ref{fig:high_right} to \ref{fig:low_wrong}, illustrate how uncertainty estimates serve as a critical calibration signal, enabling the system to distinguish between reliable judgments and potential distribution shifts.

\textbf{1. High Score, Low Uncertainty (Figure \ref{fig:high_right}).} 
This figure represents the ideal scenario where the policy model generates a correct, high-quality reasoning step that aligns well with the PRM's training distribution. The PRM confidently assigns a high score with low variance. In this case, both standard and uncertainty-aware methods correctly prioritize this step for extension.

\textbf{2. High Score, High Uncertainty (Figure \ref{fig:high_wrong}).}
Here, the step is actually incorrect, yet the PRM assigns it a high score, likely due to "hallucinated" correctness or superficial plausibility (e.g., correct-looking LaTeX formatting but flawed logic). Crucially, the \textit{high epistemic uncertainty} reveals the PRM's lack of confidence in this OOD input. While a standard greedy search would be misled into selecting this step, our UATS framework uses the high variance signal to downweight its priority or trigger additional sampling, effectively filtering out this potential trap.

\textbf{3. Low Score, High Uncertainty (Figure \ref{fig:low_right}).}
This is the most critical scenario demonstrating the value of our method. The step is factually correct but phrased in an unconventional way (OOD), causing the PRM to assign a low score. However, the associated \textit{high uncertainty} indicates that the PRM is unreliable in this region. Unlike standard methods that would prematurely prune this valid path, UATS identifies the high uncertainty and allocates additional compute budget (or applies an optimism bonus), successfully "rescuing" the correct path and allowing the reasoning to proceed to the correct solution.

\textbf{4. Low Score, Low Uncertainty (Figure \ref{fig:low_wrong}).}
This figure shows a clear reasoning error. The PRM correctly identifies it as incorrect with a low score, and the low uncertainty confirms this judgment is robust (in-distribution error). Both baseline and uncertainty-aware methods efficiently prune this step, conserving computational resources for more promising branches.

\begin{figure}
    \centering
    \includegraphics[width=0.6\linewidth]{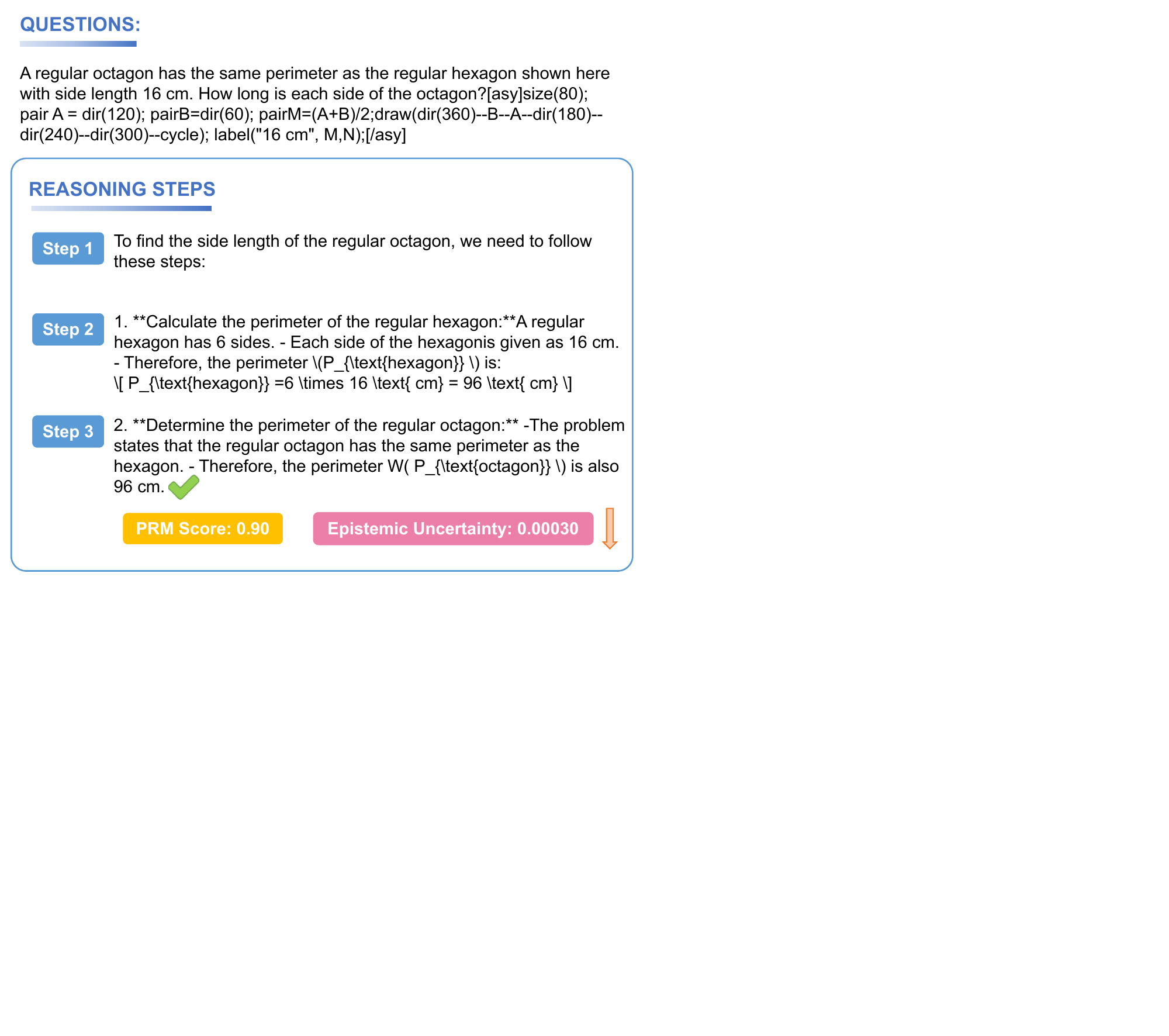}
    \caption{Reasoning step with high PRM score but low epistemic uncertainty.}
    \label{fig:high_right}
\end{figure}

\begin{figure}
    \centering
    \includegraphics[width=0.6\linewidth]{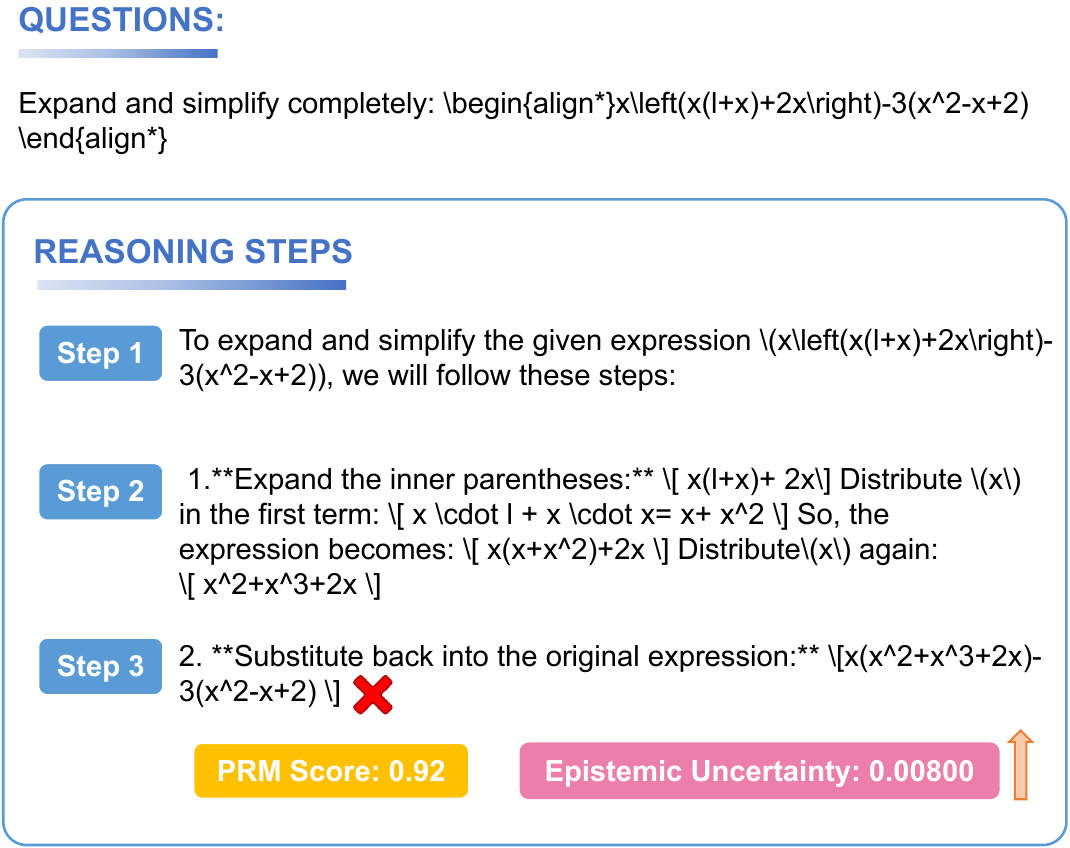}
    \caption{Reasoning step with high PRM score and high epistemic uncertainty.}
    \label{fig:high_wrong}
\end{figure}

\begin{figure}
    \centering
    \includegraphics[width=0.6\linewidth]{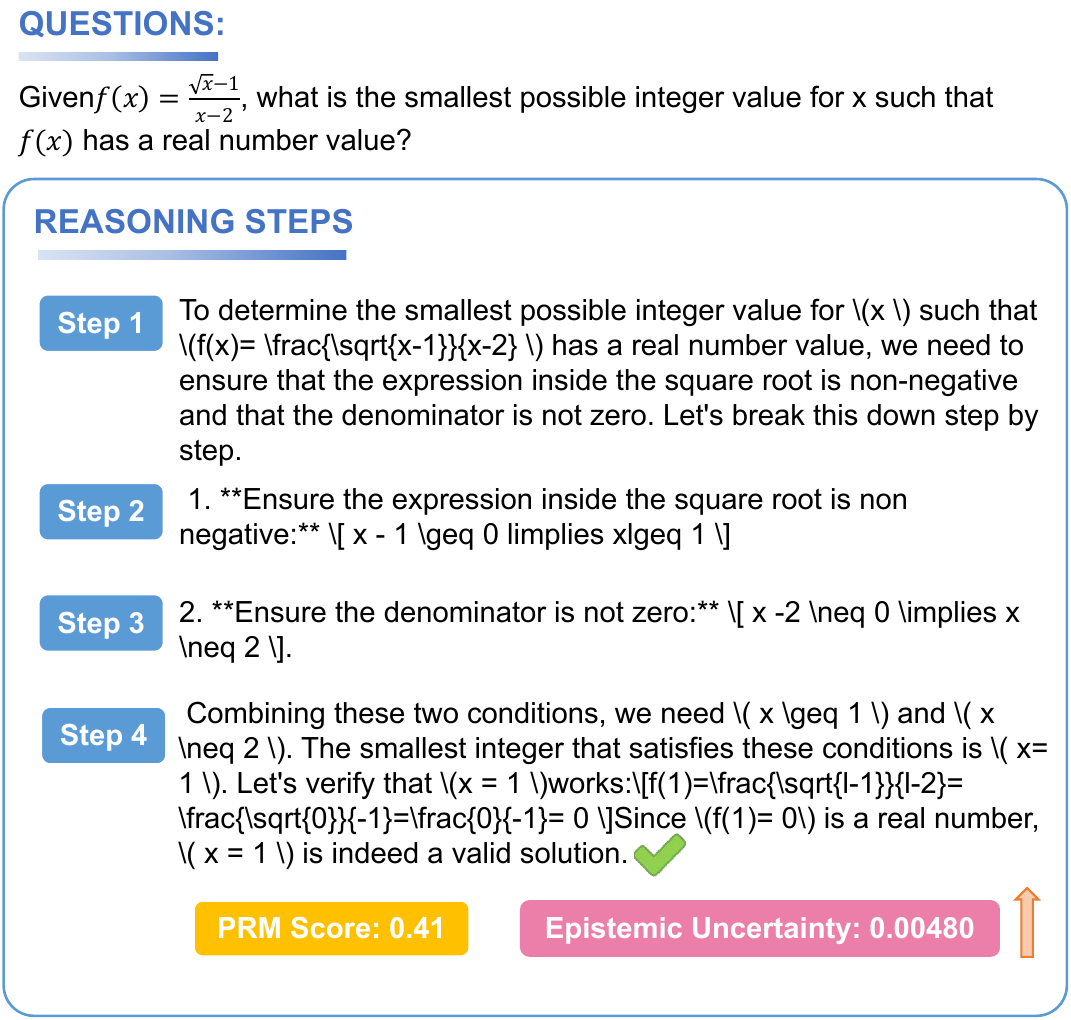}
    \caption{Reasoning step with low PRM score but high epistemic uncertainty.}
    \label{fig:low_right}
\end{figure}

\begin{figure}
    \centering
    \includegraphics[width=0.6\linewidth]{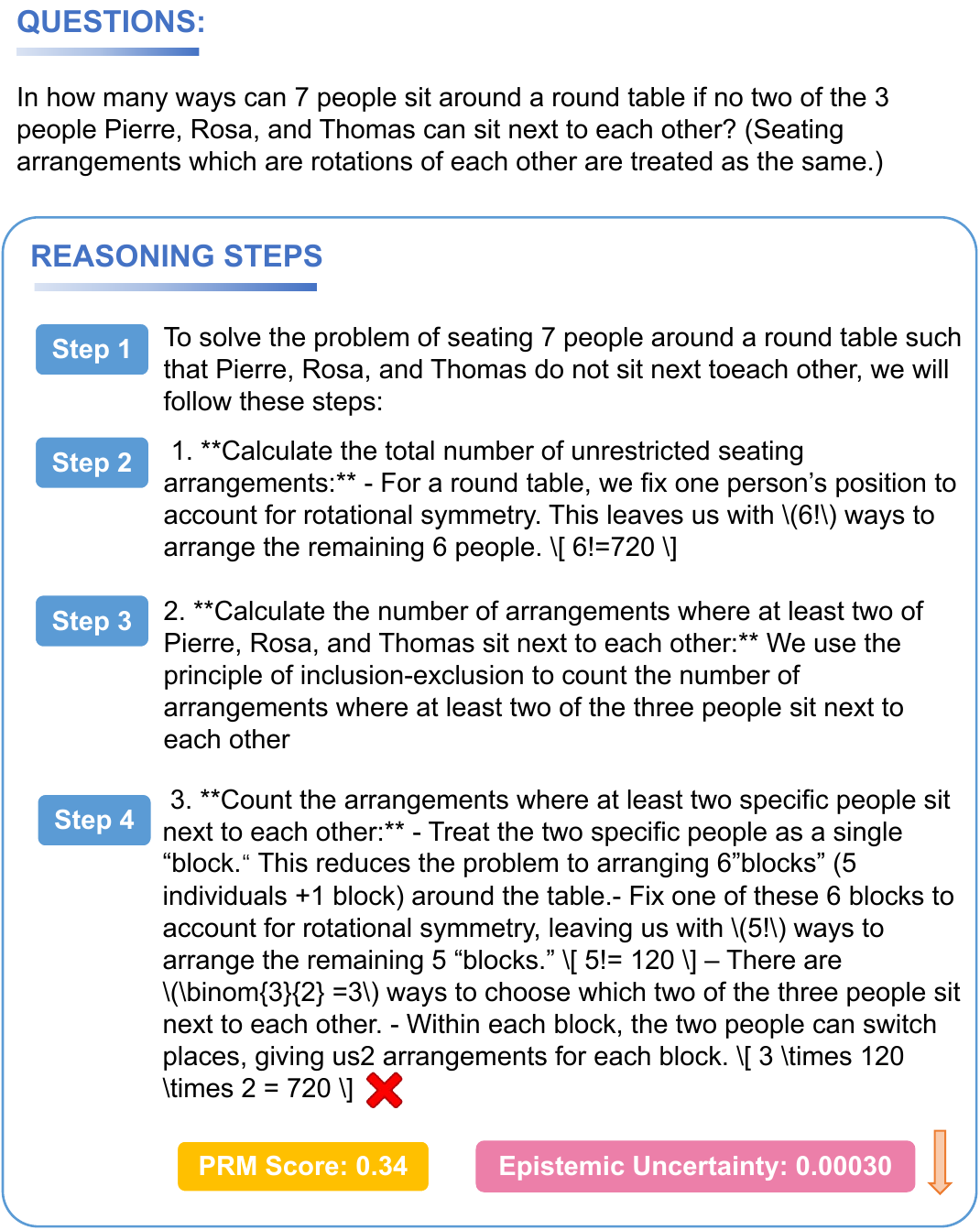}
    \caption{Reasoning step with low PRM score and also low epistemic uncertainty.}
    \label{fig:low_wrong}
\end{figure}

\end{document}